\documentclass[letterpaper, 10 pt, conference]{ieeeconf}  

\IEEEoverridecommandlockouts                              

\usepackage{times}

\usepackage{amsfonts}
\usepackage{color}
\usepackage{color}
\usepackage{graphicx}
\usepackage{amsmath}
\usepackage{amssymb}
\usepackage[table]{xcolor}

\newcommand{\Cfree}{\ensuremath{\mathcal C_{\text{free}}}}
\newcommand{\Cobs}{\ensuremath{\mathcal C_{\text{obs}}}}
\newcommand{\Cspace}{\ensuremath{\mathcal C\text{-space}}}
\newcommand{\Ccont}{\ensuremath{\mathcal C_\text{cont}}}

\newcommand{\ACcont}{\ensuremath{\widetilde{\mathcal C}_\text{cont}}}

\newcommand{\closure}{\operatorname{cl}}
\newcommand{\interior}{\operatorname{int}}

\newcommand{\q}{\ensuremath{\mathbf q}}

\newcommand{\qc}{\ensuremath{\mathbf q_c}}

\newcommand{\qa}{\ensuremath{\mathbf q_0}}
\newcommand{\dist}{\operatorname{dist}}
\newcommand{\argmin}{\operatorname{argmin}}
\newcommand{\PDt}{\ensuremath{PD_t}}
\newcommand{\PDg}{\ensuremath{PD_g}}

\usepackage[ruled, vlined,linesnumbered]{algorithm2e}
\SetAlFnt{\small}
\usepackage[caption=false]{subfig}

\newtheorem{theorem}{Theorem}

\title{\LARGE \bf
Efficient Penetration Depth Computation between Rigid Models using Contact Space Propagation Sampling
}

\author{Liang He$^{1}$ and Jia Pan$^{2}$ and Danwei Li$^{1}$ and Dinesh Manocha$^{1}$
\thanks{*This research is supported in part by ARO Contract W911NF-14-1-0437 and NSF award 1305286.}
\thanks{$^{1}$Liang He, Danwei Li and Dinesh Manocha are with the Department of Computer Science, the University of North Carolina at Chapel Hill 
        }%
\thanks{$^{2}$Jia Pan is with the Department of Mechanical Engineering and Biomedical Engineering, the City University of Hong Kong. 
        }%
}

\begin{document}

\maketitle
\thispagestyle{empty}
\pagestyle{empty}

\begin{abstract}
We present a novel method to compute the approximate global penetration depth (PD) between two non-convex geometric models. Our approach consists of two phases: offline precomputation and run-time queries.
In the first phase, our formulation uses a novel sampling algorithm to precompute an approximation of the high-dimensional contact space between the pair of models.
As compared with prior random sampling algorithms for contact space approximation, our propagation sampling considerably speeds up the precomputation and yields a high quality approximation. At run-time, we perform a nearest-neighbor query and local projection to efficiently compute the translational or generalized PD. We demonstrate the performance of our approach on complex 3D benchmarks with tens or hundreds of thousands of triangles, and we observe significant improvement over previous methods in terms of accuracy, with a modest improvement in the run-time performance. 
\end{abstract}

\section{Introduction}
\label{sec:intro}
Accurate and efficient computation of inter-penetration is important in many areas, including computer graphics, haptics, and robotics.
A common metric that is used to measure the extent of inter-penetration between two intersecting objects is \emph{penetration depth} (PD), which is defined as the minimum amount of movement or transformation required to separate two in-collision objects. The resulting motion may correspond to translational alone (translational PD) or to both translational and rotational motion PD (generalized PD).
PD computation is frequently used for many applications, such as physically-based simulation~\cite{Baraff:1998:LSC}, sample-based motion planning~\cite{Zhang:2007:GPD}, haptics~\cite{Wang:CBO:2012,Je:2012:PRP}, and contact manipulation~\cite{Koval}.

Computing the exact PD in 3D is a challenging task because of the $\mathcal O(m^3n^3)$ time complexity involved in translational PD and the $\mathcal O(m^6n^6)$ worst case time complexity for generalized PD, where $m$ and $n$ are the number of triangles in two non-convex input models~\cite{Zhang:2007:GPD}. Given the high combinatorial complexity of exact PD computation, many approximate algorithms have been proposed.
Some of the simplest algorithms compute the intersecting features of these two models and use them to compute local PD that is based on a measure of separating those overlapping features. In fact, current game engines such Box2D~\cite{Erin:2012:Box2D} and Bullet~\cite{Erwin:2012:Bullet} use local PD computations for collision response. However, the accuracy of local PD algorithms depends on relative configuration of two objects~\cite{Heidelberger04,Redon:2006:AFM}. Other techniques are based on computing an approximation of the configuration space boundary~\cite{Pan:2013:EPD,Kim:2002:FPD,Lien:2009:ASM}, but the accuracy of these techniques can vary for different configurations of two objects and it is hard to derive tight bounds. There are no good reliable algorithms for global PD computation between arbitrary non-convex 3D shapes.

\noindent {\bf Main Results}: In this paper, we present a novel algorithm to approximate global PDs between rigid objects based on efficient sampling in the contact space. Our approach can compute both translational and generalized PD with high accuracy for non-convex models. We first precompute an approximation of the contact space of two overlapping objects by generating samples in the contact space. We generate our initial samples using random sampling and use a novel propagation algorithm to generate additional samples via local search. The use of propagation sampling considerably speeds up  precomputation  and results in a high quality contact space approximation. At run-time, our algorithm performs a nearest-neighbor query to compute the PD. We also analyze the  properties of our sampling scheme and highlight its benefits. Compared with prior PD algorithms, our approach offers the following benefits:
\begin{itemize}
\item The overall algorithm is general and directly applicable to complex non-convex and non-manifold models.It can compute translational and generalized PDs. 
\item The use of propagation sampling can considerably accelerate the precomputation and provides a high quality approximation of the contact space.
\item The run-time query is very fast (a few milliseconds) and can be used for interactive applications.
\item The overall algorithm is more accurate as compared to prior local and global PD computation algorithms. 
\end{itemize}
We highlight the performance of our algorithm on different models, which contain tens or hundreds of thousands of triangles with sharp features. We also highlight the considerable improvements in the accuracy of the run-time query compared with recent algorithms based on active learning~\cite{Pan:2013:EPD} and local optimization~\cite{Tang:2014:IGP,Je:2012:PRP}. In particular, our approach can considerably reduce the error in PD computations over these methods. This paper is an extension of our work~\cite{He:2016:RAL} and provides additional technical details that were not included in~\cite{He:2016:RAL} due to the space limitations.

The remainder of this paper is organized as follows. In Section~\ref{sec:related}, we survey the literature related to the configuration space and PD computation. We introduce our notation and give an overview of the algorithm in Section~\ref{sec:overview}. We present our contact space sampling algorithm in Section~\ref{sec:method}, and we discuss and analyze its performance on many complex benchmarks in Section~\ref{sec:result} and Section~\ref{sec:analysis}.

\section{Related Work}
\label{sec:related}

\subsection{Configuration Space Computation}
There is extensive work on configuration space computations in robotics, geometric computing, and related areas. In the most general cases, configuration space computations can be reduced to computing the arrangement of contact surfaces~\cite{Varadhan:2006:TPA}. However, these approaches are susceptible to robustness issues. Moreover, the worst-case complexity of the arrangement computation can be as high as $\mathcal O(n^k)$, where $n$ is the number of contact surfaces in the arrangement and $k$ is the dimension of the configuration space~\cite{Goodman:Rourke:1997}. Some techniques for approximating the configuration space in lower dimensions are based on generating a discrete number of slices~\cite{Sacks:SCS:1997}. 
When the object movement is limited to translational motion, the resulting configuration space corresponds to the Minkowski sum of two objects~\cite{Leonidas:CCRS:1987,LPT:SpatialPlanning:1983}.

A substantial amount of work in motion planning involves approximating the configuration space with sampling techniques. These include various randomized algorithms that compute roadmaps for collision-free path planning. Some of these approaches, such as~\cite{Ji:2000:ORS, Liang:2014:ICRA, Amato:1998:OOP,Salzman:2013:MPM}, also consider the problem of sampling in a constrained configuration space or a manifold, which is similar to the contact space sampling discussed in this paper. However, these methods are designed for collision-free motion planning, and hence only require the generated samples to capture the connectivity of the free part of the constrained configuration space.

\subsection{PD Computation}

Given two convex polytopes, we can compute the exact translational PD using Minkowski sum computation~\cite{Gino:2001:GDC,Agarwal:2000:CPD,Kim:2002:DEEP}. For non-convex polyhedral models, the PD can be computed using a combination of convex decomposition, pairwise Minkowski sums, and exact union computation~\cite{Kim:2002:FPD}. Different techniques have been proposed to approximate the boundary of the Minkowski sum~\cite{Kim:2002:FPD,Varadhan:2006:TPA}, but they are limited to offline and non-interactive applications. 
Most practical algorithms for translational PD are based on local computations. These local algorithms only consider the intersecting or overlapping of features such as the vertices, edges, and faces. Based on pairwise intersections, they tend to estimate a motion that would separate these intersecting features~\cite{Guendelman:2003:NRB,Redon:2006:AFM,Lien:2009:ASM,Tang:2009:IHD,Tang:2012:CPF,weller2009inner}. 
Other techniques estimate the local intersection volume and its derivative to perform volume-based repulsion~\cite{Wang12}. 
Local translational PD computation can also be estimated using distance fields~\cite{Heidelberger04}. Point-based Minkowski sum approximation~\cite{Lien:2008:CMS} has been used approximate the translational PD. 
The exact computation of generalized PD can be formulated in terms of computing the arrangement of contact surfaces ~\cite{Zhang:2007:GPD}. However, no practical algorithms are known for exact computation due to its high combinatorial complexity. Most practical algorithms are based on local optimization techniques ~\cite{Nawratil:2009:GPD,Zhang:2007:AFP,Je:2012:PRP,Tang:2014:IGP}.
However, due to the high time and storage complexity, most generalized PD algorithms are based on local optimization-based
techniques~\cite{Nawratil:2009:GPD,Zhang:2007:AFP,Je:2012:PRP,Tang:2014:IGP}. 
~\cite{Pan:2013:EPD} recently proposed a learning-based approximate penetration depth computation algorithm that reduces the contact space problem to robust classification by finding a separating surface between in-collision and collision-free samples in the configuration space. However, this algorithm cannot provide high quality approximation of the contact space  of objects with sharp features, because it represents the contact space using the SVMs (support vector machines). Recently, Kim et al.~\cite{Kim:2015:IROS} present a hybrid PD computation algorithm that combines this active learning approach with local optimization based methods to improve its accuracy.

\section{Background and Overview}
\label{sec:overview}
In this section, we introduce our notation and give an overview.

\subsection{Contact Space}
We denote the configuration space for a pair of triangular meshes $A$ and $B$ as $\Cspace$. Each configuration or point in the configuration space represents the relative transform (i.e., position and orientation) of $A$ with respect to $B$. In the rest of this paper, we assume that $A$ is movable and $B$ is fixed. The configuration space is composed of two parts: collision-free space represented as $\Cfree = \closure(\{\q: A(\q) \cap B = \emptyset\})$, and in-collision or obstacle space represented as $\Cobs = \interior(\{\q: A(\q) \cap B \neq \emptyset\})$, where $A(\q)$ corresponds to $A$ located at the configuration $\q$, and $\closure(\cdot)$ and $\interior(\cdot)$ correspond to set closure and interior operations, respectively.

The boundary of $\Cfree$ is called the \emph{contact space} and is denoted as $\Ccont =\partial \Cfree$. The contact space corresponds to the configurations where $A$ and $B$ just touch each other without any penetration. Moreover, a contact configuration is classified as a collision-free configuration in our formulation.

\subsection{PD Formulation}
\label{sec:overview:pdformulation}

The global penetration depth corresponds to the minimum motion or transformation required to separate two intersecting objects $A$ and $B$~\cite{Agarwal:2000:CPD,Kim:2002:DEEP}:
\begin{align}
\label{eq:2:PDgdef} \text{PD}(A(\qa), B) = \min_{\q \in
\Ccont} \dist(\qa, \q),
\end{align}
where $\qa$ corresponds to an in-collision configuration and $\q$ is a configuration that belongs to the contact space $\Ccont$.
We use the notation $\dist(\cdot, \cdot)$ to represent a distance metric between two configurations. This includes the Euclidean metric for translational PD, and many different formulations can be used for generalized PD computation. We denote $\qc$ as the contact configuration where PD$(A, B)$ achieves its
minimal value: $\qc = \argmin_{\q \in \Ccont} \dist(\qa, \q)$.

Different formulations of PD can be defined by appropriate $\dist(\cdot, \cdot)$ metrics. The metric for the translational motion ($\PDt$) is simple and is the standard Euclidean distance metric between vectors corresponding to the configurations. The metric for the general motion ($\PDg$) can be defined using different formulations, including the weighted Euclidean distance~\cite{Wang:CBO:2012}, object norm~\cite{Kazerounian:ASME:1992,Je:2012:PRP,Tang:2014:IGP}, and a displacement distance metric~\cite{Zhang:2007:AFP}. In this paper, we use the object norm~\cite{Kazerounian:ASME:1992,Je:2012:PRP,Tang:2014:IGP} as the $\PDg$ metric, which can be intuitively defined as an average squared length of all displacement vectors between two objects.

\begin{figure*}[!ht]
\centering
\includegraphics[width=0.95\linewidth]{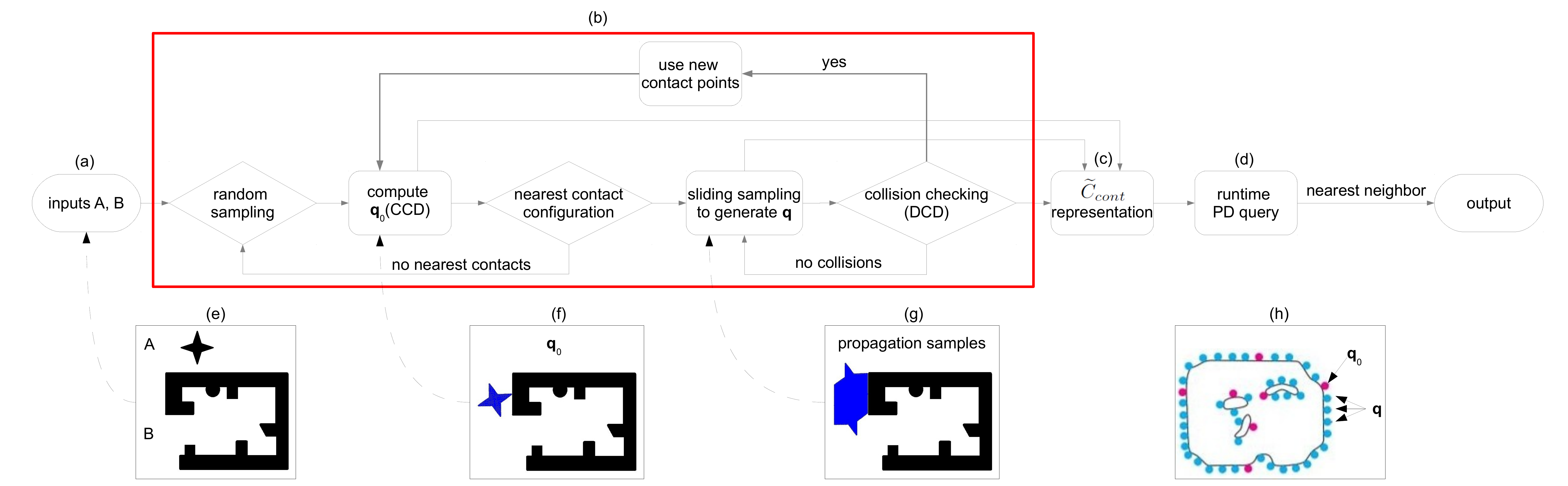}
\caption{The offline computation pipeline and the run-time phase of our algorithm. Given two input objects in (a), the precomputation algorithm in (b) performs the propagation sampling to efficiently generate an approximation to the contact space ($\ACcont$) as the output (c). This approximate contact space is then used for efficient run-time PD query as shown in (d). (e) shows the shapes of the input objects. (f) and (g) show the relative transform between two input objects under the configuration $\q_0$ randomly generated as the propagation seed and a set of configurations generated from $\q_0$ using local-search based propagation, respectively. (h) is the output $\ACcont$ where red points are the random seed samples and the blue points are the propagate-samples.}
\label{fig:pipeline}
\end{figure*}

\subsection{Approximate $\Ccont$ Computation}

In order to construct an approximation of the contact space, we perform an offline sampling in the configuration space, as shown in Figure~\ref{fig:pipeline}. Given two input objects $A$ and $B$,  our method starts with a \emph{random-sample} in the contact space. This random-sample can be generated using traditional continuous collision checking (CCD) techniques. These techniques compute the first time of collision or contact between two objects by reducing the problem to finding roots of polynomial equations corresponding to the triangle features. 
Continuous collision detection has been widely used for physically-based simulation~\cite{Tang:2012:CPF} as well as local planning in robotics. 
Next, our method performs an iterative local search around the initial random-sample by sliding object $A$ over object $B$'s surface, and generates more samples on the contact space. We denote the samples generated on the contact space during the local search as \emph{propagate-samples} to distinguish them from random-samples (see Figure~\ref{fig:pipeline}(h)). The local search stops when no more propagate-samples can be generated, and we then restart the iteration with a new random-sample in the contact space. This iterative process continues until a sufficient number of samples have been generated. The random-samples and propagate-samples computed by our approach make up an approximate sample-based representation of the contact space $\Ccont$ between $A$ and $B$, and we denote this approximation as $\ACcont$.

\subsection{Approximate PD Computation}
Given the approximate representation of the contact space, $\ACcont$, we compute the approximate global PD by performing a nearest-neighbor query in $\Ccont$. The definition of approximate penetration depth is analogous to the exact penetration depth in Equation~\ref{eq:2:PDgdef}:
\begin{align}
\label{eq:PDgdef} \overline{\text{PD}}(A(\qa), B) = \min_{\q \in \ACcont}\dist(\qa, \q),
\end{align}
where the domain for $\q$ is restricted to $\ACcont$.
The accuracy of $\overline{\text{PD}}$ is governed by the accuracy of  $\ACcont$ with respect to $\Ccont$.
Given a query configuration $\qa$, we perform a nearest-neighbor search to find the configuration $\qc$ that is closest to the decision boundary $\ACcont$. Finally, the distance between $\qa$ and $\qc$ is computed using an appropriate distance metric $\dist(\cdot, \cdot)$ and the result is an approximation of the exact PD value. As mentioned in Section~\ref{sec:overview:pdformulation}, we use the object norm as the distance metric.

\section{Contact Space Propagation Sampling}
\label{sec:method}
In this section, we present our contact space propagation sampling algorithm that computes an offline approximation of the contact space.
Our sampling algorithm is an iterative algorithm. During each iteration, we start from a random-sample in the contact space, and then perform a local search around this initial sample to generate more propagate-samples on the contact space. Once the local search stops, we repeat the iterative step with a new random-sample. This iterative process continues until a sufficient number of samples have been generated. The random-sample on the contact space during each iteration is computed by first generating two samples in the configuration space, one in collision-free space and the other in the in-collision space. We join those samples by a straight line in the configuration space and find its intersection with the contact space. This reduces to computing the first time of contact between a collision-free and an in-collision configuration, which corresponds to a CCD query. The resulting sample on the contact space is the random-sample used during this iteration.

Ideally, the local search procedure should run many steps and generate sufficient numbers of propagate-samples, in order to cover a large portion of the contact space. This coverage is important for the efficiency of the sampling algorithm, because the generation of a random-sample requires the expensive CCD query. This query is more expensive than the generation of a propagate-sample that only needs to perform the DCD (discrete collision detection) query. If the local search can generate a high number of propagate-samples, it amortizes the computational cost of generating a random-sample over a high number of propagate-samples, and improves the efficiency of our precomputation step.
Our main goal is to design a fast and effective local search algorithm that can compute the propagate-samples quickly. To that end, we perform a breadth-first propagation on the contact space, starting from the random-sample. The breadth-first propagation maintains a queue (we call it the \emph{propagate-queue}) of contact samples. During each step of this propagation, we pop one contact sample $\q$ from this queue, and then slide object $A(\q)$ over the surface of object $B$ in different directions along its boundary. This propagation step results in a set of new samples $\{\q'\}$ around $\q$, as shown in Figure~\ref{fig:pipeline}. Next, we perform collision checking for these samples $\{\q'\}$, and only add the collision-free samples into the propagate-queue. The breadth-first search is repeated until the queue is empty.

A new sample $\q'$ is propagated from a contact sample $\q$ and is added into the propagate-queue only if it is collision-free. Otherwise, the sample would be discarded and the actual execution of the local search's propagation may be interrupted, as shown in Figure~\ref{fig:propagate}. To overcome this challenge, we classify the local search process into two cases: the {\em boundary configuration case} when $\q' \in \Ccont$, and the {\em internal configuration case} when $\q' \in \Cobs$. In the internal configuration case, we resume the local search computation according to objects' contact features (i.e., local vertices, edges, faces). This formulation greatly improves the efficiency of the local search. The overall local search computation algorithm is shown in Algorithm~\ref{algo:localsearch}.

\subsection{Boundary Configuration Case}
In the boundary configuration case, each step of the local search is a standard propagation step from a contact sample $\q$. We first compute the contact pair $(p_A, p_B)$ between $A(\q)$ and $B$, where $p_A$ and $p_B$ are two contact points on objects $A$ and $B$. The points $p_A$ and $p_B$ belong to two different objects but overlap with each other, just touching at the configuration $\q$.  We also compute the angle $\theta$ between the normal vectors at $p_A$ and $p_B$. Next, we slide the object $A$ over the surface of object $B$ with a distance $d$, which can be any value less than the edge length of triangular mesh(since we generate the samples on the mesh of object).
For our case, we we use the edge lengths of the fixed objects as the sliding step, and thus the propagate-samples all land in the vertices of object $B$. During the sliding movement, the contact point on $A$ remains unchanged as $p_A$, and the contact point on $B$ moves from $p_B$ to $p_B'$ (see Figure~\ref{fig:propagate}). Now $A$ and $B$ touch at the new contact point pair $(p_A, p_B')$. We further rotate the object $A$ such that the angle between the two objects' contact normals remains to be $\theta$ (as shown in Figure~\ref{fig:propagate}). In this way, we compute a new configuration $\q'$, which can be specified using $p_A$, $p_B'$ and $\theta$ for 2D objects. Similar propagation procedure can also be defined for 3D objects, whose configuration space has 6 dimensions.

\begin{algorithm}[!h]
  \SetKwInOut{Input}{input}\SetKwInOut{Output}{output}
  \Input{Two objects $A$ and $B$, an initial random-sample $\q^r$, the search step size $d$}
  \Output{A set of  propagate-samples $S$ from $\q^r$ }
  \BlankLine
  \tcc{Initialize final result set}
  $S \leftarrow \emptyset$ \;
  \tcc{Initialize a propagate-queue $Q$}
  $p_A$, $p_B$ $\leftarrow$ contact points of $A(\q^r)$ and $B$ \;
  $\theta \leftarrow$ angles between contact normals \;
  $Q \leftarrow \{(\q^r, p_A, p_B, \theta)\}$ \;
  \While{$Q \neq \emptyset$}
        {
          $(\q, p_A, p_B, \theta)$ $\leftarrow$ \text{pop}$(Q)$ \;
          $N \leftarrow$ $B$'s vertices at a step of $d$ away from $p_B$ \;
          \For{$p_B' \in N$}
              {
                $(\q', p_A, p_B', \theta) \leftarrow \mathcal{T}(\q, p_A, p_B, \theta)$ \;
                \If{\emph{isCollision}$(\q')$}
                   {
                     \tcc{Internal configuration case}
                     Compute the critical $\q^c$ between $\q$ and $\q'$ \;
                     $M \leftarrow$ contact pairs other than $(p_A, *)$ between $A(\q^c)$ and $B$ \;
                     \For{$(p_A^c, p_B^c) \in M$}
                         {
                           $\theta^c \leftarrow$ angles between contact normals at $p_A^c$ and $p_B^c$ \;
                           \tcc{Check whether $\q^c$ is close to previous samples}
                           \If{\emph{kdTreeTest}($S$,$\q^c$,$r$) = true}
                           {
                            continue \;
                           }
                           $Q \leftarrow Q \cup \{(\q^c, p_A^c, p_B^c, \theta^c)\}$ \;
                         }
                   }
                \Else
                   {
                     \tcc{Boundary configuration case}
                      $Q \leftarrow Q \cup \{(\q', p_A, p_B', \theta)\}$ \;
                   }

              }
              $S \leftarrow S \cup \{\q\}$ \;
        }
  \caption{Local search for propagate-samples}\label{algo:localsearch}
\end{algorithm}

We represent the sliding movement from $\q$ to $\q'$ as a transition function $(\q', p_A, p_B', \theta) = \mathcal{T}(\q, p_A, p_B, \theta)$.
The new generated configuration $\q'$ is  pushed into the propagate-queue for future propagation-sample computations. This sliding procedure is executed along different directions on the surface of object $B$ around $p_B$, and results in a set of new configurations $\{\q'\}$ spreading over the neighborhood of $\q$ on the contact space. The collision-free samples in $\{\q'\}$ are located on the contact space and we add them directly into the propagate-queue. For in-collision samples in $\{\q'\}$, we treat them as the internal configuration case and stop propagation.

\begin{figure}[h]
\centering
\includegraphics[width=0.85\linewidth]{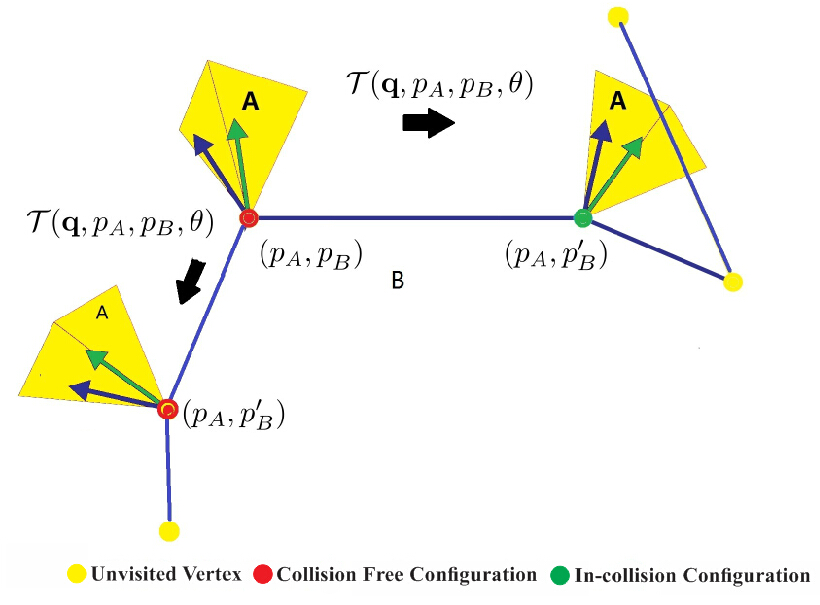}
 \vspace*{-0.1in}
\caption{
An example of the contact sample propagation between an object $A$ (the yellow tetrahedron) and an object $B$ (the blue polygon). The object $A$ is initially at configuration $\q$ where the contact points between two objects are $p_A$ and $p_B$ and the angle between contact normals is $\theta$. During the propagation, the object $A$ moves along one edge of the object $B$'s mesh and finally contacts with object $B$ at $p_B'$, a neighbor vertex of $p_B$. Object $A$ always has point $p_A$ as its contact point. }
\vspace*{-0.1in}
\label{fig:propagate}
\end{figure}

In the above description, we restrict all new propagate-samples $\{\q'\}$ to have the same angle $\theta$ between the two objects' contact normals. This simple heuristic greatly increases the probability that a new sample $\q'$ will be collision-free. It is based on the assumption that the surface curvature around the neighborhood of $p_B$ is roughly constant. Therefore, a configuration $\q'$ with the same relative angle as $\q$ should have a high probability to be collision-free after the sliding movement.

The parameter $d$ in a boundary configuration propagation step determines the step-size of the propagation. Its value varies during the local search process. In particular, $d$ is inversely proportional to the relative scale of $A$ with respect to $B$, and it is also related to the surface curvature at $p_B$.

\subsection{Internal Configuration Case}
\label{sec:method:interrupt}
In this case, $\q$  is inside $\Cobs$, i.e. inside the C-obstacle space.
 This can happen when two objects are very close in size or the surface of $B$ is `bumpy', i.e., the curvature changes dramatically over the surface. In this case, the propagation search step usually explores only a few steps because $A$ and $B$ will collide even when $A$ only slides a small step over $B$'s surface. In an extreme situation, every local search returns no propagation-samples and all samples generated on the contact surface are random-samples.
This will result in very slow sampling procedure, and these samples cannot be evenly distributed over the contact space.

Our solution for the internal configuration case exploits the contact features, and is based on the following property of the sliding movement:
\begin{theorem}
\label{thm:slide}
{\em Suppose one step of slide moves a contact sample $\q^0$ into an in-collision sample $\q^1$. The transition function is
\begin{equation}
\label{eq:transit}
(\q^1, p_A, p_B^1, \theta) = \mathcal{T}(\q^0, p_A, p_B^0, \theta),
\end{equation}
where $(p_A, p_B^0)$ is the contact pair at $\q^0$, $(p_A, p_B^1)$ is the contact pair at $\q^1$, and $\theta$ is the angle between contact boundary configurations. On the resulting sliding trajectory, there exists one configuration $\q^t$ such that $A(\q^t)$ and $B$ are in contact, but have at least one additional contact point other than $p_A$. Here we assume the sliding movement is parameterized by $t \in [0,1]$. We call the configuration $\q^t$ the critical configuration.}
\end{theorem}
\begin{proof}
For each $p \in A$, we denote $p(\q)$ as its position corresponding to the configuration $\q$. We further define $c = \inf_{p \in A - \{p_A\}} \inf_{s \in [0,1]} \{s \mid p(\q^s) \cap B \neq \emptyset\}$, i.e., the first time that $A$ contacts with $B$ on a
 point other than $p_A$. Since $\q^0$ is collision-free and $\q^1$ is in-collision, we know $c \in (0, 1)$. $\q^c$ has the property that $A(\q^c)$ contacts with $B$, but there is at least one additional contact point other than $p_A$. Therefore, the theorem is proved.
\end{proof}

Based on this theorem, we resume the propagation at an in-collision sample $\q'$ by first computing the critical configuration $\q^c$ between $\q$ and $\q'$. $\q^c$ has more than one contact point, and we denote the set of contact points other than $p_A$ as $M$. For each contact point $p_A^c$ in $M$, we continue the local search step by changing the contact point on $A$ that remains unchanged during the movement from $p_A$ to $p_A^c$. In particular, we compute the contact pair $(p_A^c, p_B^c)$ and $\theta^c$, the angle between the contact boundary configurations at $p_A^c$ and $p_B^c$. Object $A$ now starts sliding from the contact point $p_A^c$ and the corresponding transition function is $\mathcal{T}(\q^c, p_A^c, p_B^c, \theta^c)$. Figure~\ref{fig:smooth1} illustrates this process.
\begin{figure}[!ht]
\centering
\includegraphics[width=0.8\linewidth]{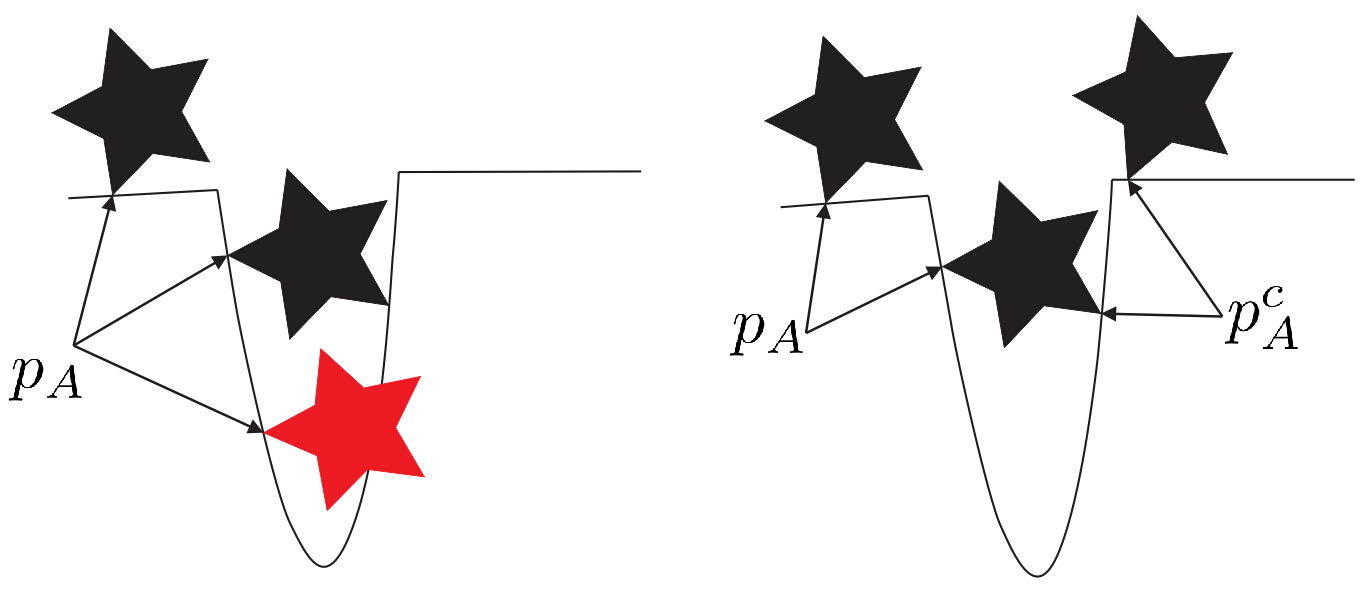}
\caption{
Objects with highly varying curvatures where the sliding movements can result in in-collision configurations and the early termination of the propagation procedure. After applying the change of contact points, the local search can escape from the `bumpy' surface region. In the left figure, the star shaped object $A$ moves over the surface $B$, and the red star denotes the sample configuration corresponding to the internal configuration case. In the right figure, we compute the critical configuration $\q^c$ and find a new contact point $p_A^c$. }
\label{fig:smooth1}
\vspace*{-0.1in}
\end{figure}

To avoid generating repeated samples during the internal configuration case, we use a kd-tree to check the new contact configuration $\q^c$ and use it as the new propagation seed only if it is not close to any existing samples, as shown in the line $15$ in Algorithm~\ref{algo:localsearch}.

\subsection{Run-time PD Queries}
In the PD query stage, we use the contact configurations generated during the precomputation. Given a query configuration $\q \in \mathcal C_{obs}$, we use nearest-neighbor query to find
two contact configurations $\q_c^0$ and $\q_c^1$ that are closest to $\q$ and incident to a same triangle on the fixed object. Next, we compute a configuration $\q_c$ which is a linear combination of these two contact configurations, and the line connecting $\q^2$ an $\q$ is perpendicular to the line connecting $\q_c^0$ and $\q_c^1$:
\begin{equation}
\q^2 = (1-\rho) \q_c^0 + \rho \q_c^1; \ \  \mbox{and} \ \  (\q^{2}-\q) \perp (\q_c^1 -\q_c^0).
\end{equation}
We then perform a linear search from $\q^2$ along directions of $\q - \q^2$ until we find a collision-free configuration $\q_c^2$. The nearest-neighbor query is performed based on the $\dist()$ metrics listed in Section~\ref{sec:overview}. Finally, we have to compare the distance of $\q^2$ with those of $\q_c^0$, $\q_c^1$, and choose the smallest one as the PD value of $\q$.

\section{Performance and Comparison}
\label{sec:result}
In this section, we highlight the performance of our propagation sampling based precomputation and the run-time PD query on a set of challenging 3D benchmarks as shown in Figure~\ref{fig:benchmarks}.
We use PQP query package to perform all the collision tests used in our framework. For PD query, we investigate both the translational and generalized PD, using the distance metric as mentioned in Section~\ref{sec:overview}.

We implement our algorithm in C++ on an Intel Core i7 CPU running at 3.30GHz with 16GB of RAM on Window 7 (64-bits) PC. All the performance and timing results are generated using a single core. Our precomputation algorithm can be easily parallelized on a multi-core PC because both the random-samples and the propagate-samples on the contact space can be generated in parallel.

\subsection{Performance}
\begin{figure}[!h]
  \centering
  \subfloat[Donut]{\includegraphics[width=0.24\linewidth]{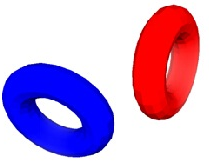}}
  \subfloat[CAD1]{\includegraphics[width=0.20\linewidth]{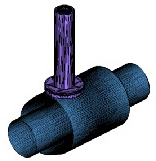}}
  \subfloat[CAD2]{\includegraphics[width=0.20\linewidth]{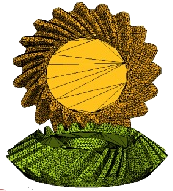}}
  \subfloat[CAD2 zoomed view]{\includegraphics[width=0.20\linewidth]{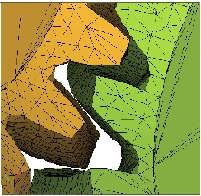}} \\
  \subfloat[Buddha]{\includegraphics[width=0.2\linewidth]{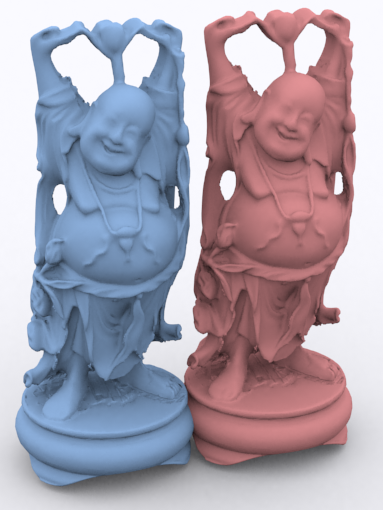}} \qquad
  \subfloat[Dragon]{\includegraphics[width=0.21\linewidth]{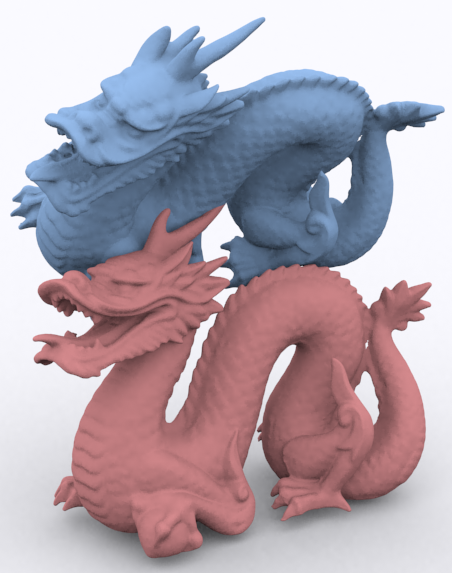}} \qquad
  \subfloat[Teeth]{\includegraphics[width=0.34\linewidth]{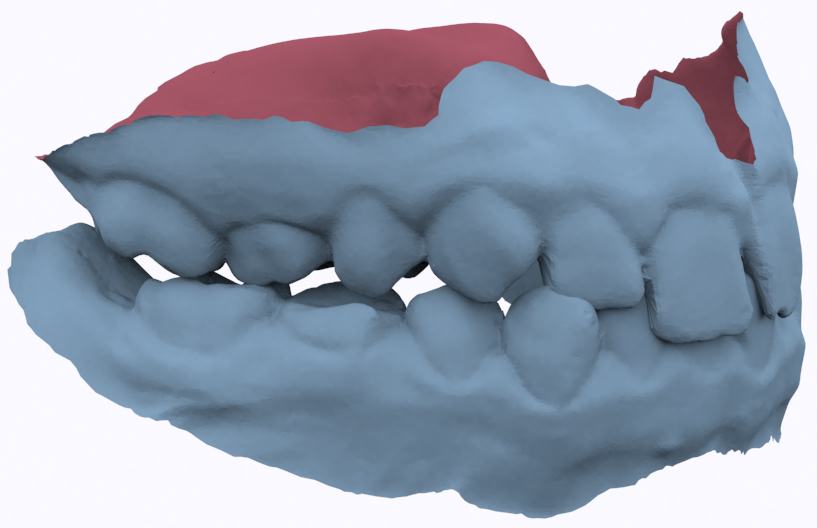}}
  \caption{The benchmarks we used to investigate the PD performance: (a) Donut, each with 576 triangles; (b) CAD1 with about 10K triangles each; (c) CAD2 with about 12K triangles each; (d) a zoomed view of CAD2 is also provided to show the complexity of this benchmark; (e) Buddha with 1M triangles each; (f) Dragon with with 230K triangles each; (g) Teeth models with about 40K triangles each.}
  \label{fig:benchmarks}
\end{figure}

\noindent {\bf Time Cost:} We now highlight the performance of both the off-line precomputation and run-time PD queries. Table~\ref{table:precomputation} and Table~\ref{table:precomputationLocal} show the time costs of the precomputation phase for general PD and translational PD respectively, given different number of propagation steps. We use at most 100,000 propagation steps, and each step will generate one random-sample and a set of propagate-samples extended from the random-sample. We can see that the time cost for a single propagation step increases while more samples are generated, and this is due to the fact that while more samples are generated, it is more difficult to find unrepeated new samples in the contact space. Table~\ref{table:precomputation} and Table~\ref{table:precomputationLocal} also provide the number of samples generated for each propagation step, which is averaged over all the 100,000 propagation steps. 
Table~\ref{table:PDQuery} provides our method's run-time cost for a single PD query, which is computed as the average time cost after computing 1,000 randomly generated in-collision queries. 

{\small
\begin{table}
\begin{tabular}{|c|cccc|c|c|}
 \hline
 \#propagations  &30,000  &50,000 & 80,000& 100,000& \#samples\tabularnewline
\hline
Donut    & {314} & {551}  & {693}  & {891} & 79 \tabularnewline
CAD1    & {421} & {569}  & {745}  & {883} & 411\tabularnewline
CAD2    & {390} & {581}  & { 737}  & {923} & 323\tabularnewline
Teeth    & {520} & {892}  & { 1,455}  & {1,621} & 1，231\tabularnewline
Dragon & {590}& {940}  & {1,401}  & {1,632} & 2，997 \tabularnewline
Buddha & {603}& {875}  & {1,337}  & {1,501} & 1，820\tabularnewline
\hline
\end{tabular}
\caption{\footnotesize{Performance of the precomputation cost for general PD on different benchmarks. We vary the number of propagation steps (i.e., the number of random-samples) from $30,000$ to $100,000$, and the corresponding time costs (in seconds) are shown in the first four columns.  We also provide the average number of samples per propagation step in the column \#samples.}}
\label{table:precomputation}
\vspace*{-0.3in}
\end{table}

}
{\small
\begin{table}
\begin{tabular}{|c|cccc|c|c|}
 \hline
 \#propagations  &30,000  &50,000 & 80,000& 100,000& \#samples\tabularnewline
\hline
Donut    & {156} & {193}  & {344}  & {423} & 63 \tabularnewline
CAD1    & {159} & {243}  & {387}  & {659} & 329\tabularnewline
CAD2    & {163} & {247}  & { 401}  & {613} & 341\tabularnewline
Teeth    & {155} & {471}  & { 912}  & {973} & 768\tabularnewline
Dragon & {333}& {752}  & {941}  & {1,429} & 667 \tabularnewline
Buddha & {346}& {771}  & {1,017}  & {1,918} & 751\tabularnewline
\hline
\end{tabular}
\caption{\footnotesize{Performance of the precomputation cost for translational PD on different benchmarks. We vary the number of propagation steps (i.e., the number of random-samples) from $30,000$ to $100,000$, and the corresponding time costs (in seconds) are shown in the first four columns.  We also provide the average number of samples per propagation step in the column \#samples.} }
\label{table:precomputationLocal}
\vspace*{-0.3in}
\end{table}

}
\noindent {\bf Convergence:} We investigate the convergence of our sampling algorithm by varying the number of samples and evaluating its benefits in terms of approximating the contact space. Since it is difficult to obtain the ground truth for the general PD, \emph{the investigation is only for the translational PD}. We use the result from the Minkowski sum based approach~\cite{Lien:2008:CMS} as the ground truth of the contact space, and then we compute the error between the PD results from the ground truth and the PD results provided by our propagation sampling approach. 

As we generate more random-samples and propagate-samples, our precomputation algorithm provides a better approximation of the contact space. To evaluate the approximation quality of the contact space, we measure the number of vertices in the original models that have been visited during the propagation as contact points (e.g., $p_A$ and $p_B$ in Figure~\ref{fig:propagate}). This is due to the fact that the denser the contact space is sampled, the more vertices should be visited as contact points during the sampling process. The measurement result is shown in Figure~\ref{fig:converge}, and we can observe that for all benchmarks, the propagation procedure convergence after generating $10$ million samples, and a high quality approximation to the contact space is achieved.
\begin{figure}[!h]
\centering
\includegraphics[width=0.95\linewidth]{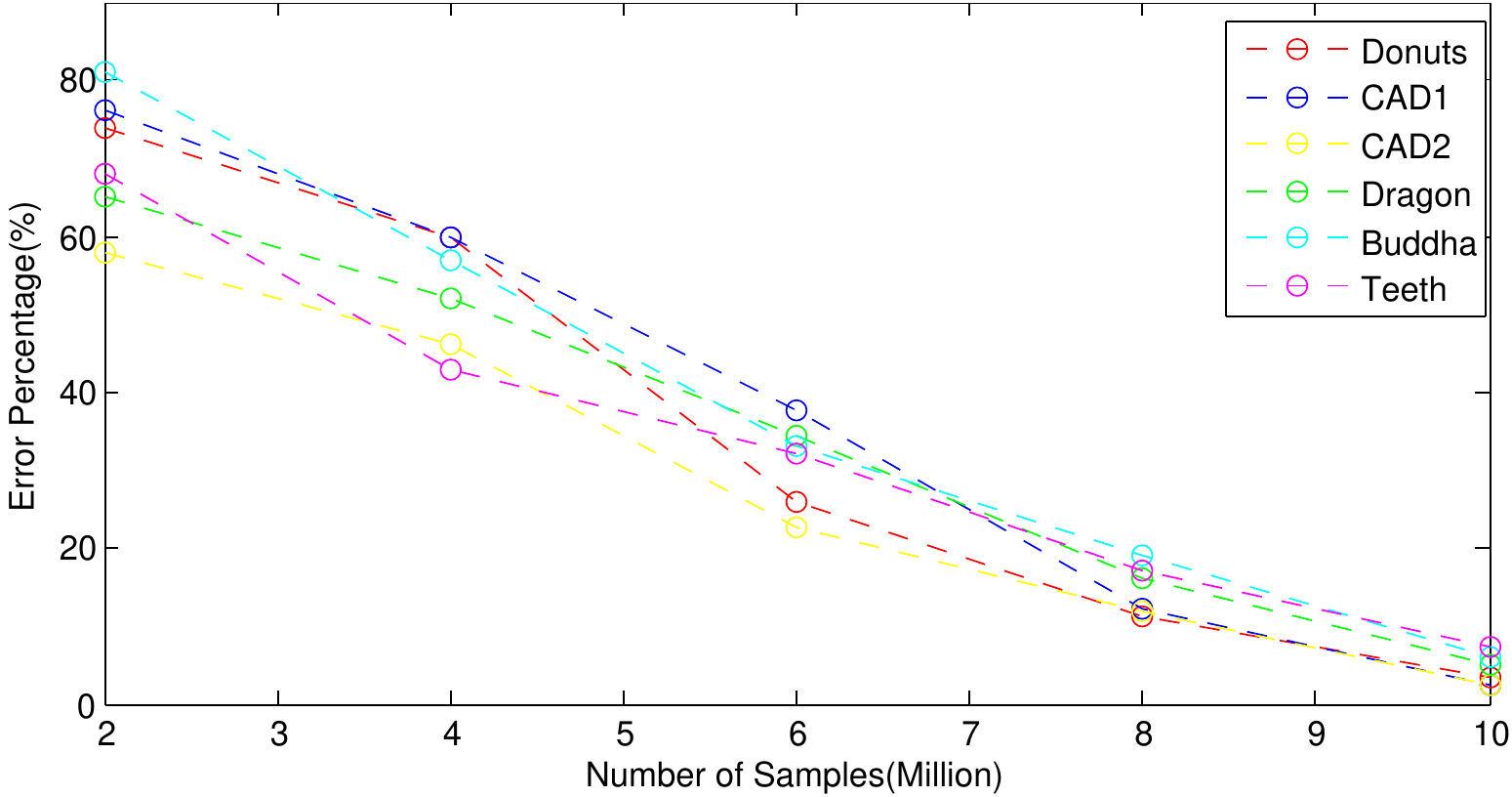}
\caption{Convergence analysis of our precomputation algorithm for general PD. The $x$-axis corresponds to the number of samples (both random and propagation). The $y$-axis corresponds to our method's error, which decreases while more samples are generated. In all cases, our algorithm converges quickly after $10$ million samples. }
\label{fig:converge}
\vspace*{-0.1in}
\end{figure}

\subsection{Comparison with Prior Methods}
We compare the performance of our algorithm with two recent methods that can compute translational and generalized PD. One is a global approach~\cite{Pan:2013:EPD} that uses active learning to generate the samples in the entire configuration space, and then computes the contact space approximation as the surface separating collision-free and in-collision samples. Its approximation of the contact surface is relatively smooth due to its leverage of SVMs and hence this method may not provide accurate PD results for query configurations with small penetrations. Overall, this approach provide a conservative bound on the PD.
The second method uses local optimization~\cite{Tang:2014:IGP,Je:2012:PRP} to compute translational and generalized PD. It starts from a good initial guess to the PD result and optimizes the computation based on appropriate metric. Its accuracy depends on the initial guess and it may get stuck in a local minimum. Unlike our method or the active learning approach, this method has no precomputation. Therefore this approach can also be used for PD computation involving deformable models. Recently, Kim et al.~\cite{Kim:2015:IROS} presented a hybrid approach that combines the global method~\cite{Pan:2013:EPD} with the local method~\cite{Tang:2014:IGP} to overcome some of these problems. 

\noindent {\bf{Runtime Performance:}} We compare the performance of runtime query of our approach with these methods, and the result is shown in Table~\ref{table:PDQuery}, for both general and translational PD formulations. We observe that our approach is significantly faster than other two approaches, because our method can generate samples more evenly and densely distributed in the contact space, and thus query configuration can easily find a nearest-neighbor contact configuration in fewer iterations.

\noindent {\bf{Accuracy of PD computation:}} In our experiments, our method quickly generates more than 1 million samples in 10-12 minutes with less than 5 MB storage. The error of our online PD query is about 5\% of the actual PD value for both the translational and general PD, while this value is 15-20\% for~\cite{Pan:2013:EPD} and 10\% for~\cite{Tang:2014:IGP,Je:2012:PRP}. The main reason for this improvement is because our method searches over the entire contact space globally while the local optimization methods~\cite{Tang:2014:IGP,Je:2012:PRP} search along the gradient direction. The gradient search strategy can only generate suboptimal results, especially for complex models. Our approach also outperforms~\cite{Pan:2013:EPD,Kim:2015:IROS} which approximates the contact space by active learning, because it can provide a more accurate approximation of the contact space.

We compare the accuracy between our method and the active learning approach and the local optimization approach by computing the per-frame error on a set of different benchmarks (CAD1, Bunny, Lion, Dragon and Donuts). The result is shown in Figure~\ref{fig:comparison}. From the results, we can observe that the PD error of our method is about 5\% for both the translational and general PD queries. The error is mainly caused by the resolution of the samples. But as we discussed in the algorithm, all the samples in our method is located on the contact space, and thus the result is much more accurate than that of~\cite{Pan:2013:EPD} which fits the contact space by machine learning method. While both methods use the same number of samples for general PD computation, the error of~\cite{Pan:2013:EPD} is about 10\% - 20\%, as shown in the third row of Figure~\ref{fig:comparison}, while our method is about 5\% as shown in the fourth row of Figure~\ref{fig:comparison}. For the translational PD, the error of our approach is about 3\%, which is smaller the the corresponding results in local optimization approaches~\cite{Tang:2014:IGP,Je:2012:PRP} as shown in the second row of Figure~\ref{fig:comparison}. Local optimization approaches also search over the objects' surface, but they do not perform well for objects with complex shapes. This is because they only search along the gradient direction and may only obtain sub-optimal solutions. For instance, they may miss the optimal results in areas of sharp features. Our method will not miss these areas in most situations since we can generate 1 million samples over the contact space, which help our method to explore the entire contact space. By searching through these samples, our method can generate more accurate results than gradient searching method. Overall, our approach is more accurate than prior PD methods.

\begin{figure*}[!htb]
\centering
\subfloat[]{\includegraphics[width=0.19\linewidth]{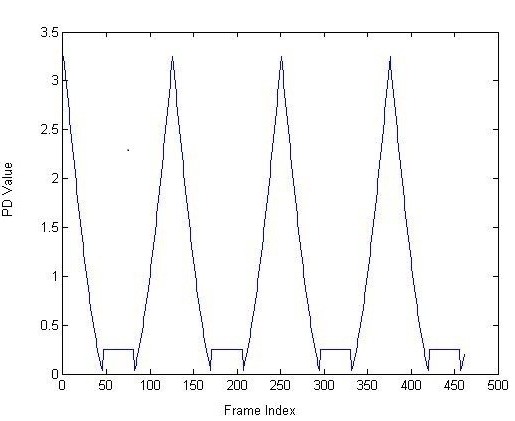}} 
\subfloat[]{\includegraphics[width=0.19\linewidth]{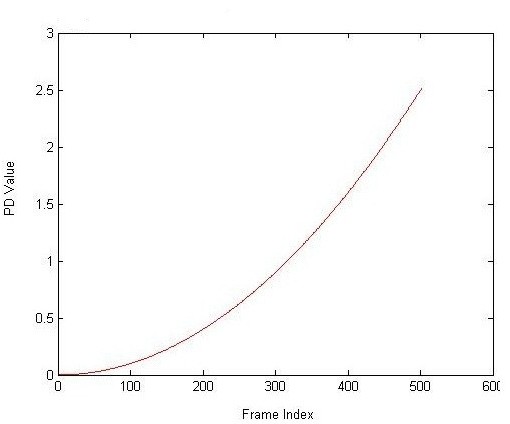}} 
\subfloat[]{\includegraphics[width=0.19\linewidth]{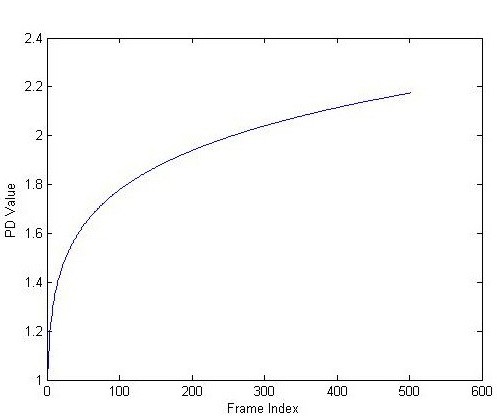}} 
\subfloat[]{\includegraphics[width=0.19\linewidth]{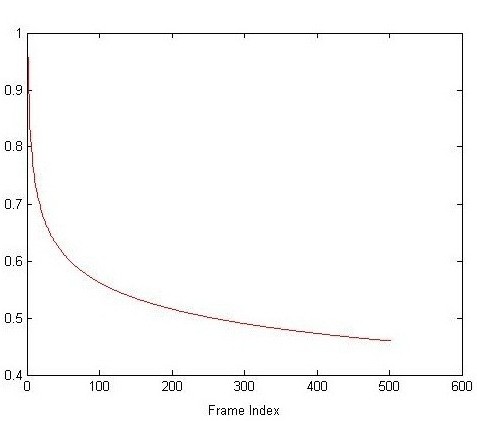}} 
\subfloat[]{\includegraphics[width=0.19\linewidth]{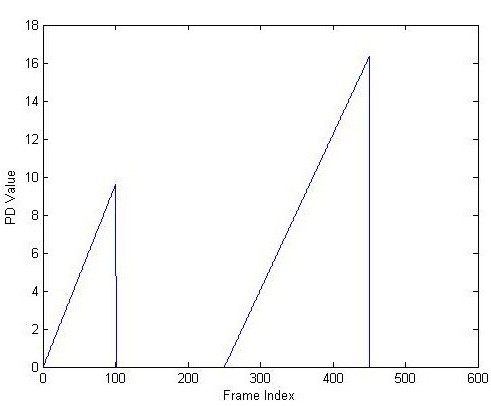}} \\
\subfloat[]{\includegraphics[width=0.19\linewidth]{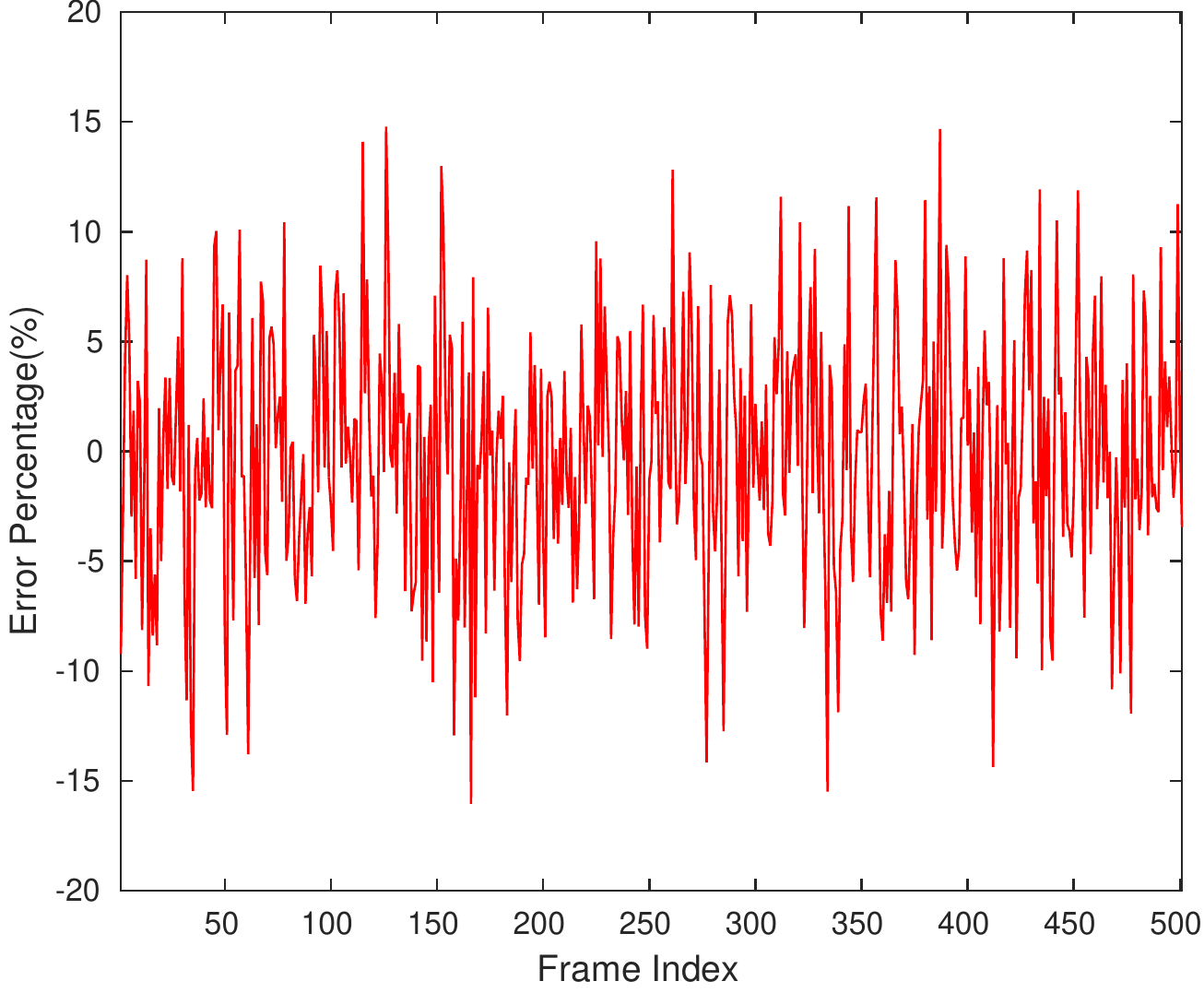}} 
\subfloat[]{\includegraphics[width=0.19\linewidth]{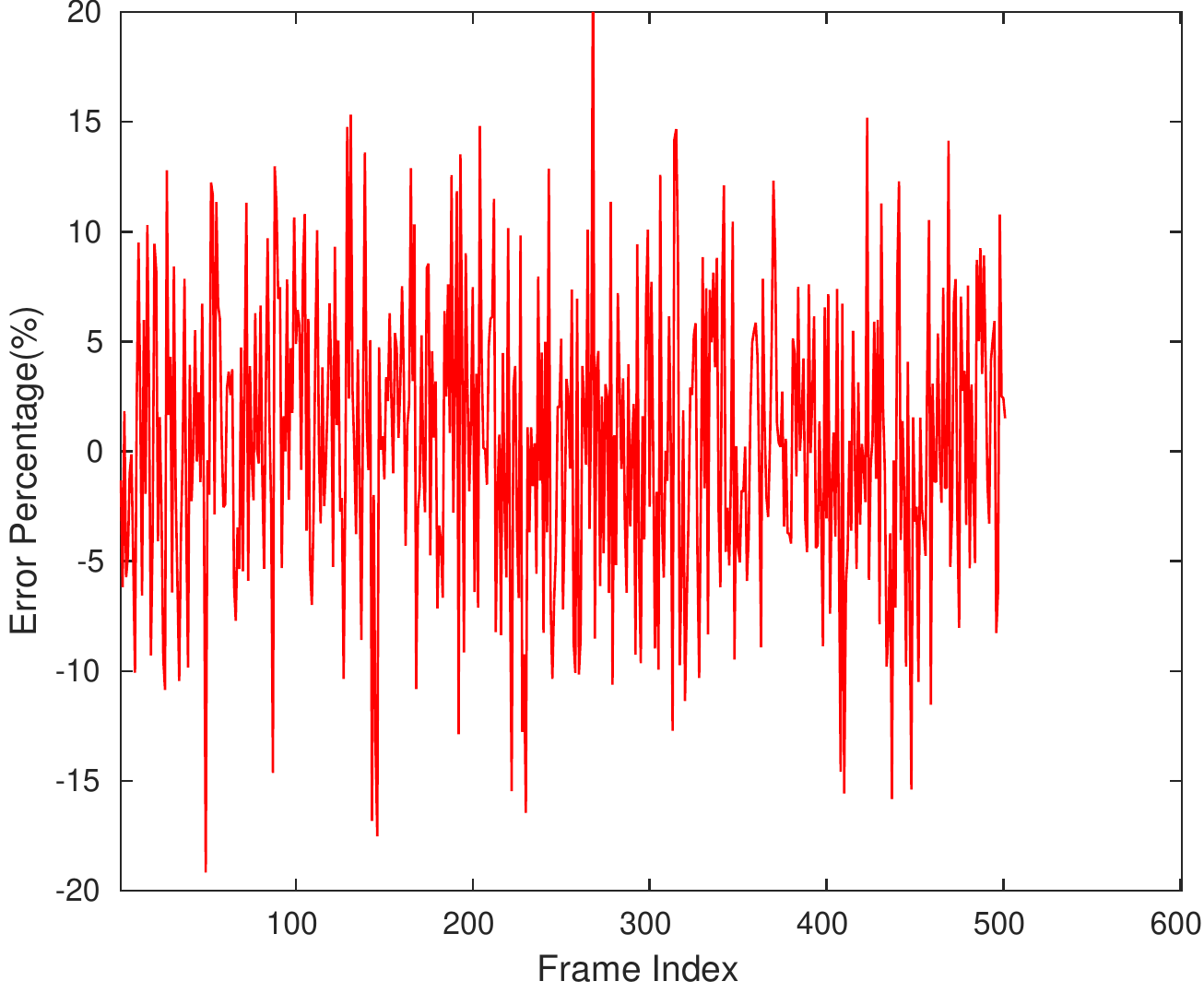}} 
\subfloat[]{\includegraphics[width=0.19\linewidth]{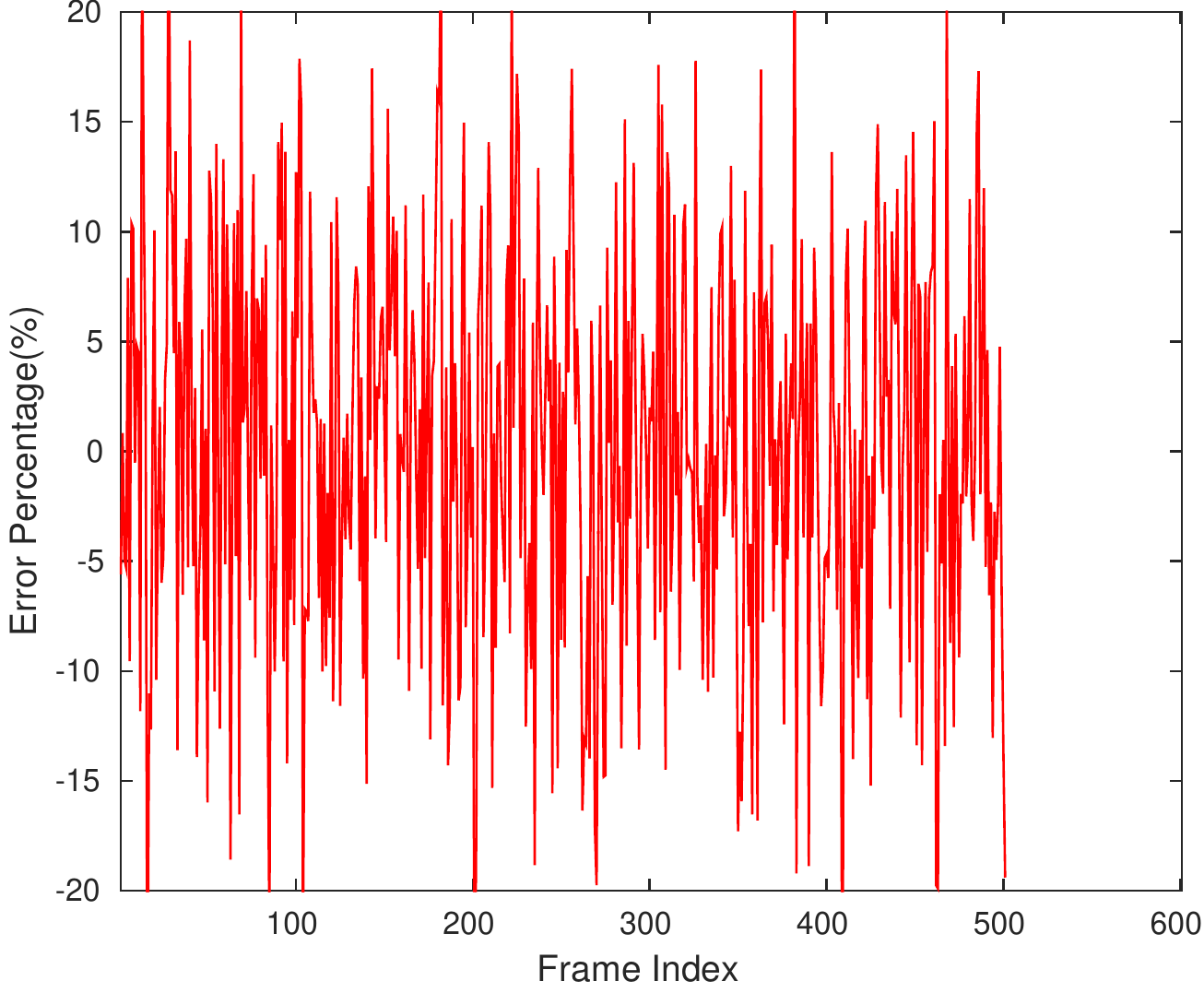}} 
\subfloat[]{\includegraphics[width=0.19\linewidth]{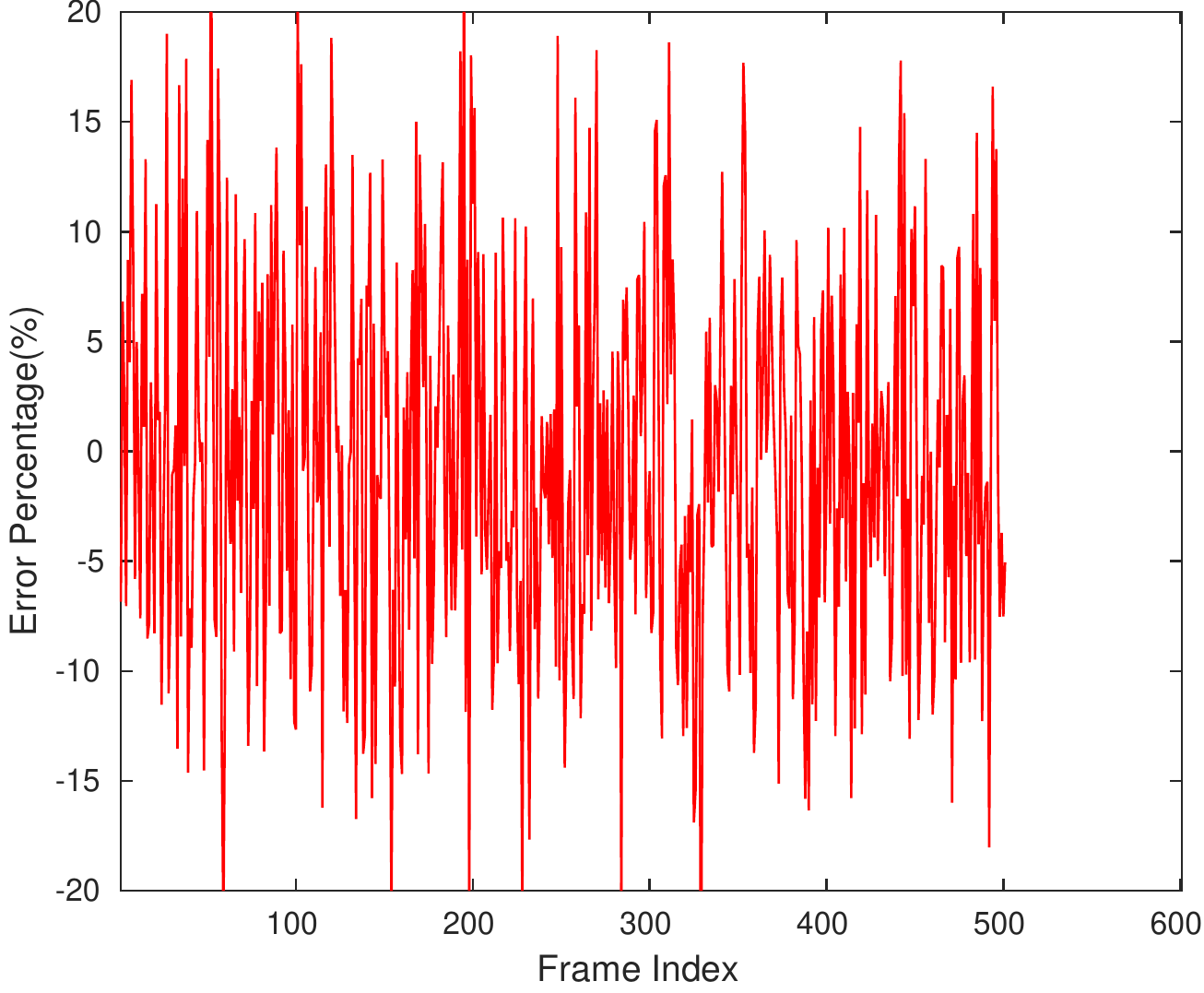}} 
\subfloat[]{\includegraphics[width=0.19\linewidth]{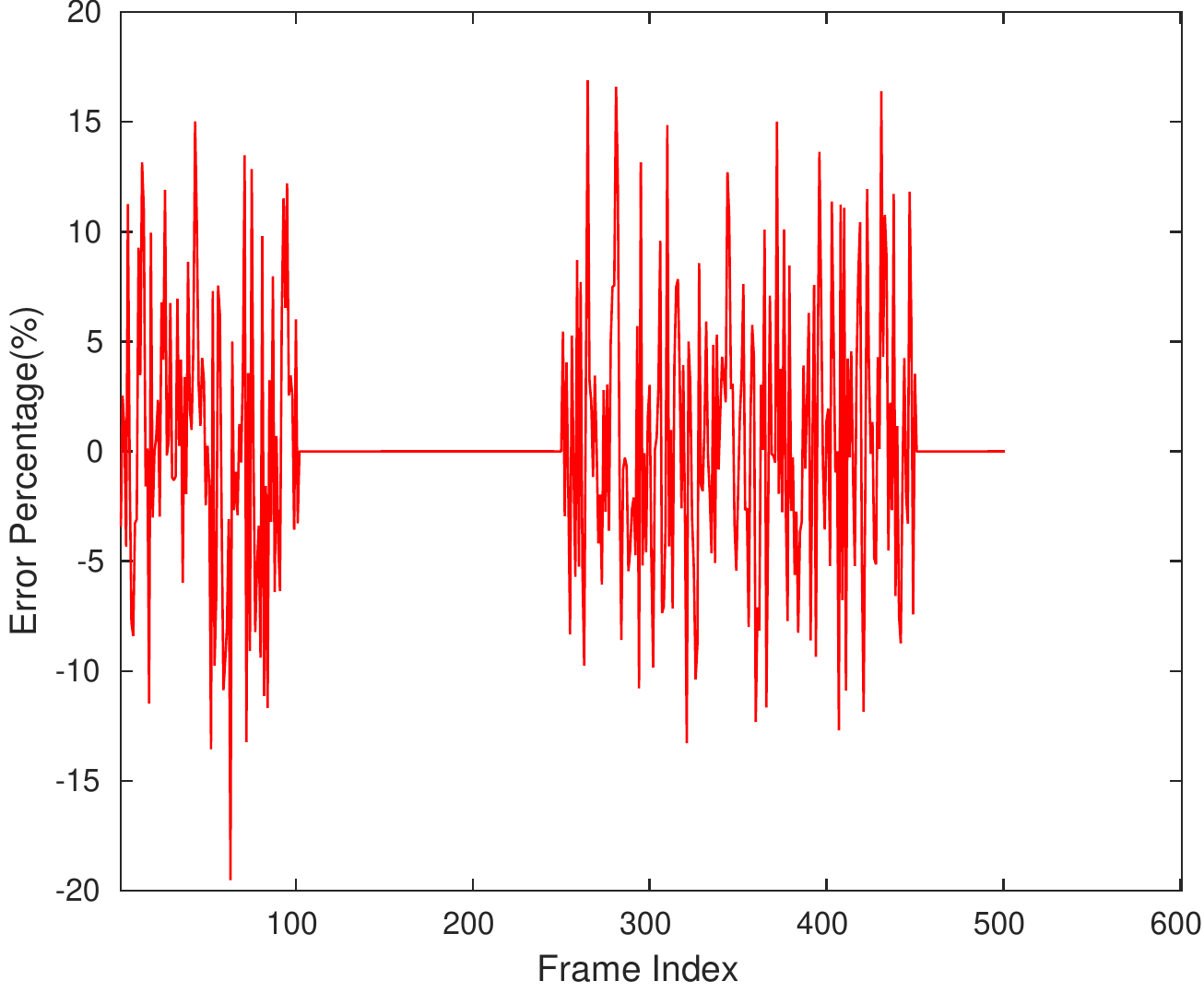}} \\
\subfloat[]{\includegraphics[width=0.19\linewidth]{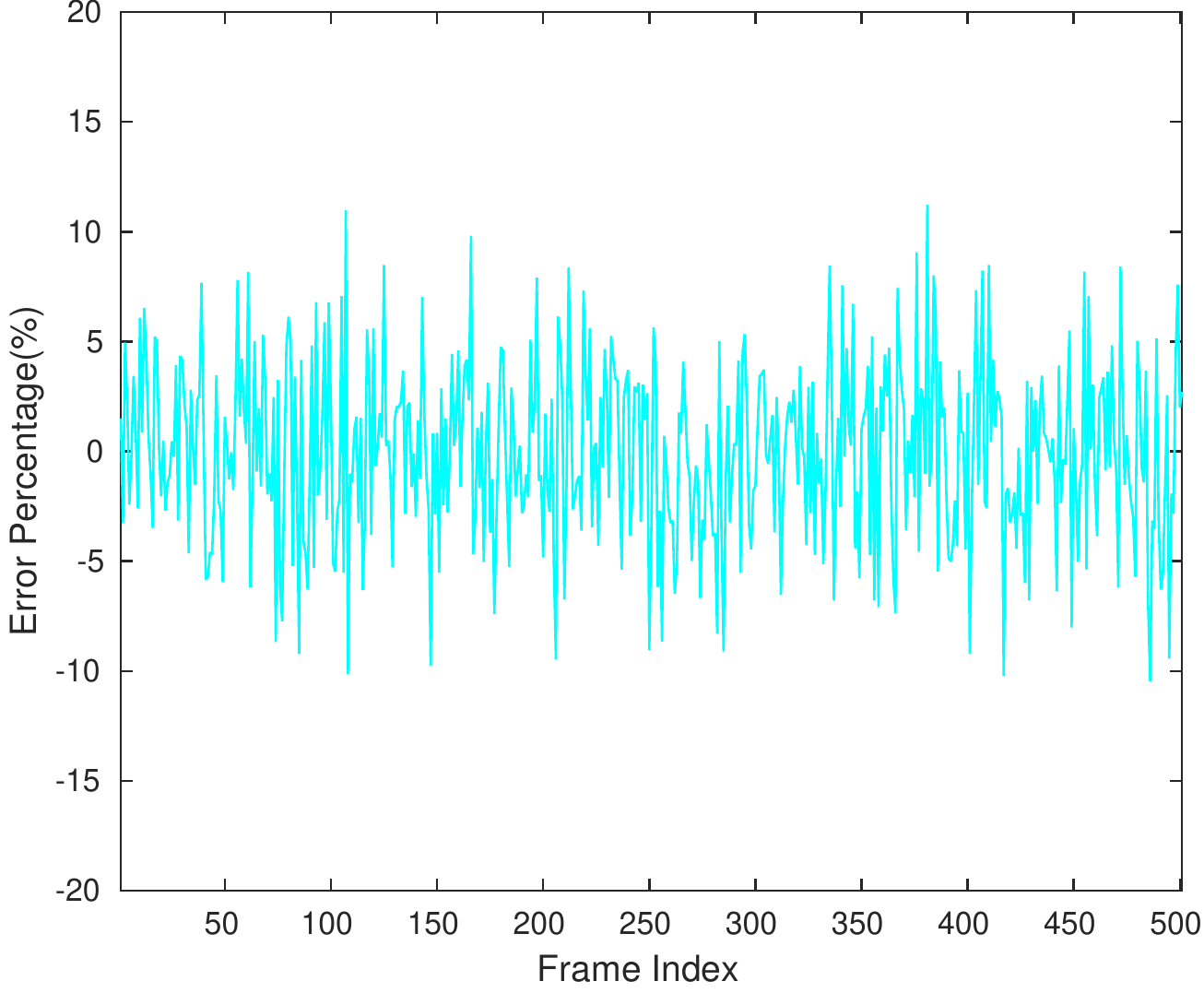}} 
\subfloat[]{\includegraphics[width=0.19\linewidth]{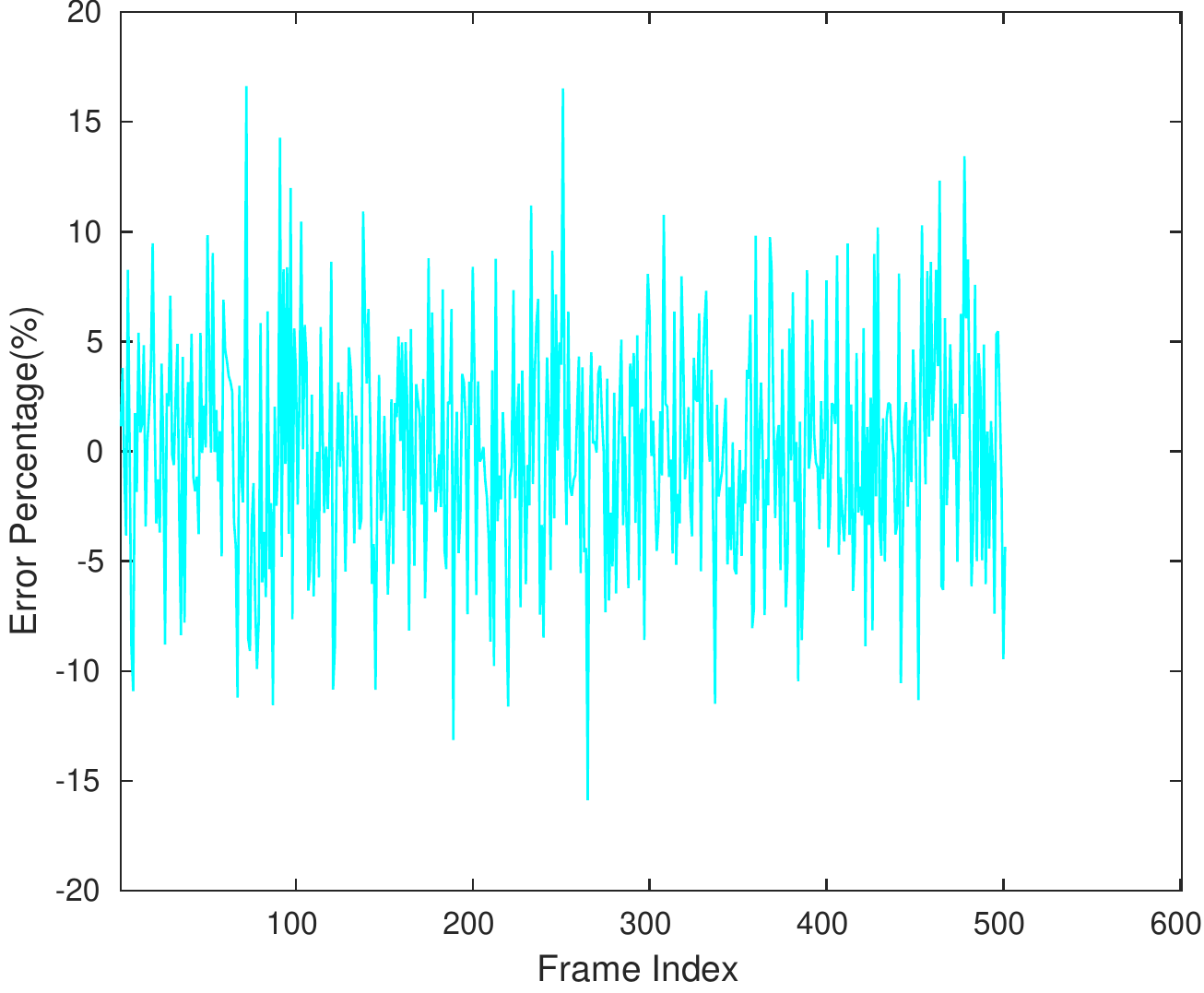}} 
\subfloat[]{\includegraphics[width=0.19\linewidth]{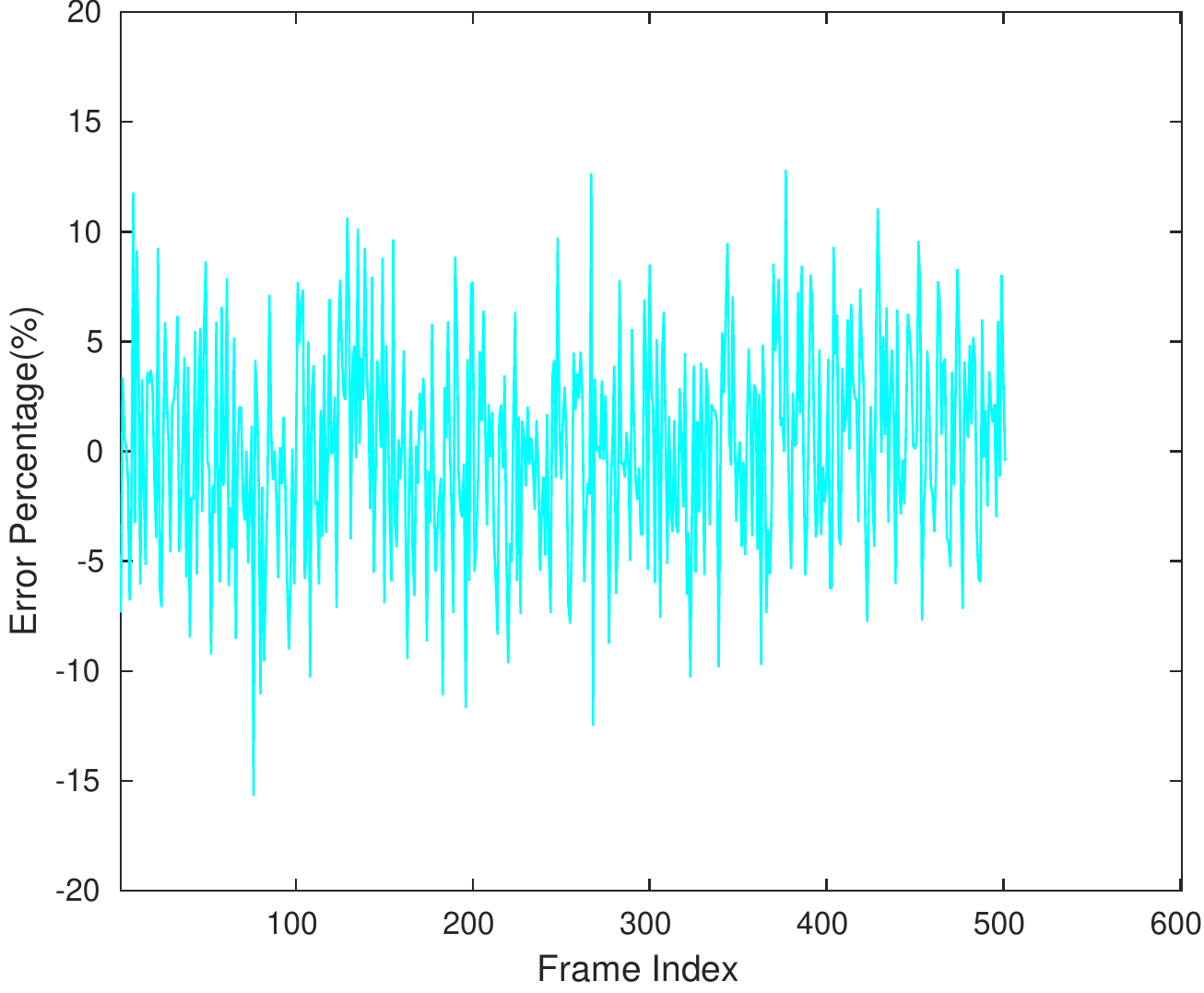}} 
\subfloat[]{\includegraphics[width=0.19\linewidth]{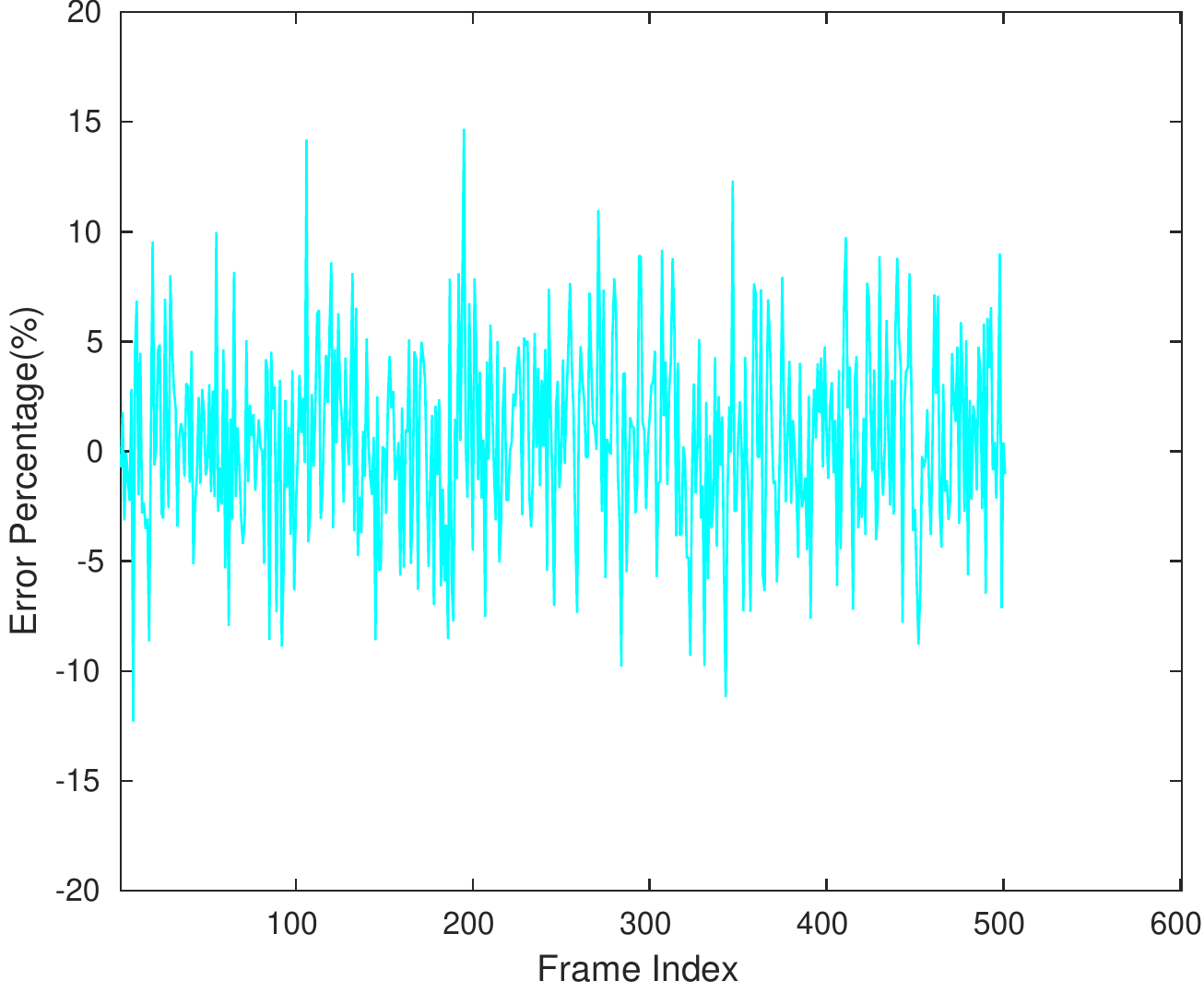}} 
\subfloat[]{\includegraphics[width=0.19\linewidth]{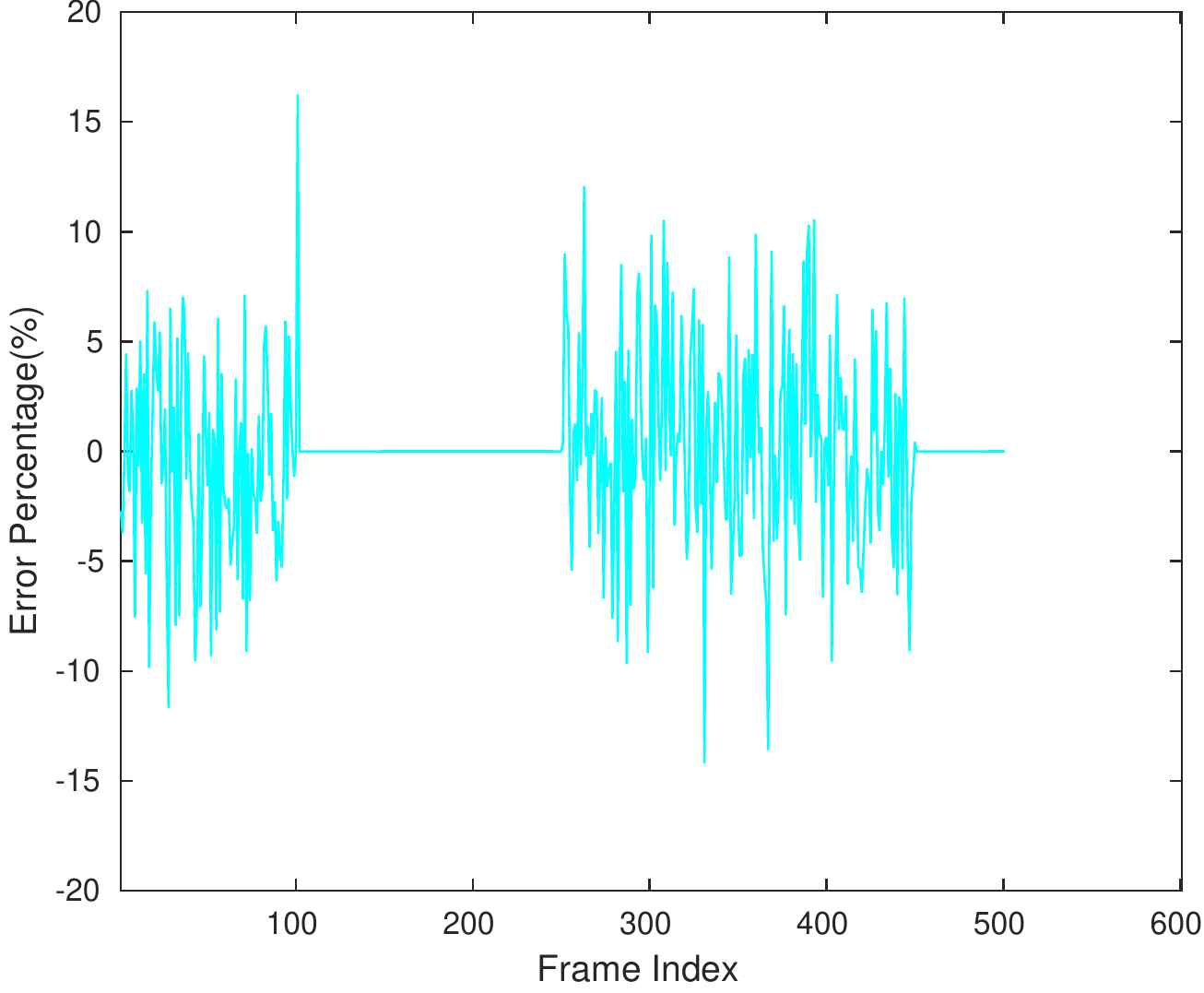}} \\
\subfloat[]{\includegraphics[width=0.19\linewidth]{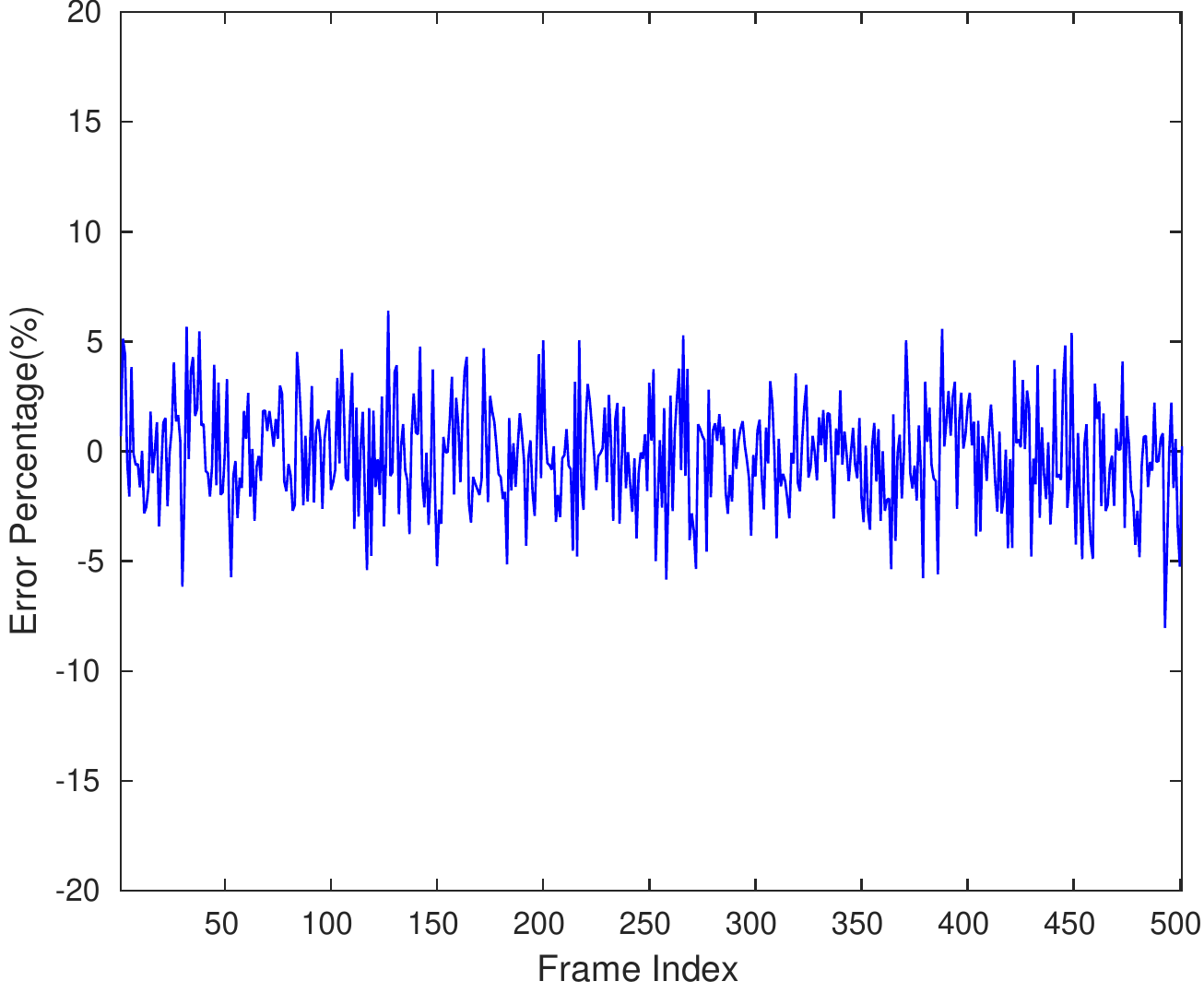}} 
\subfloat[]{\includegraphics[width=0.19\linewidth]{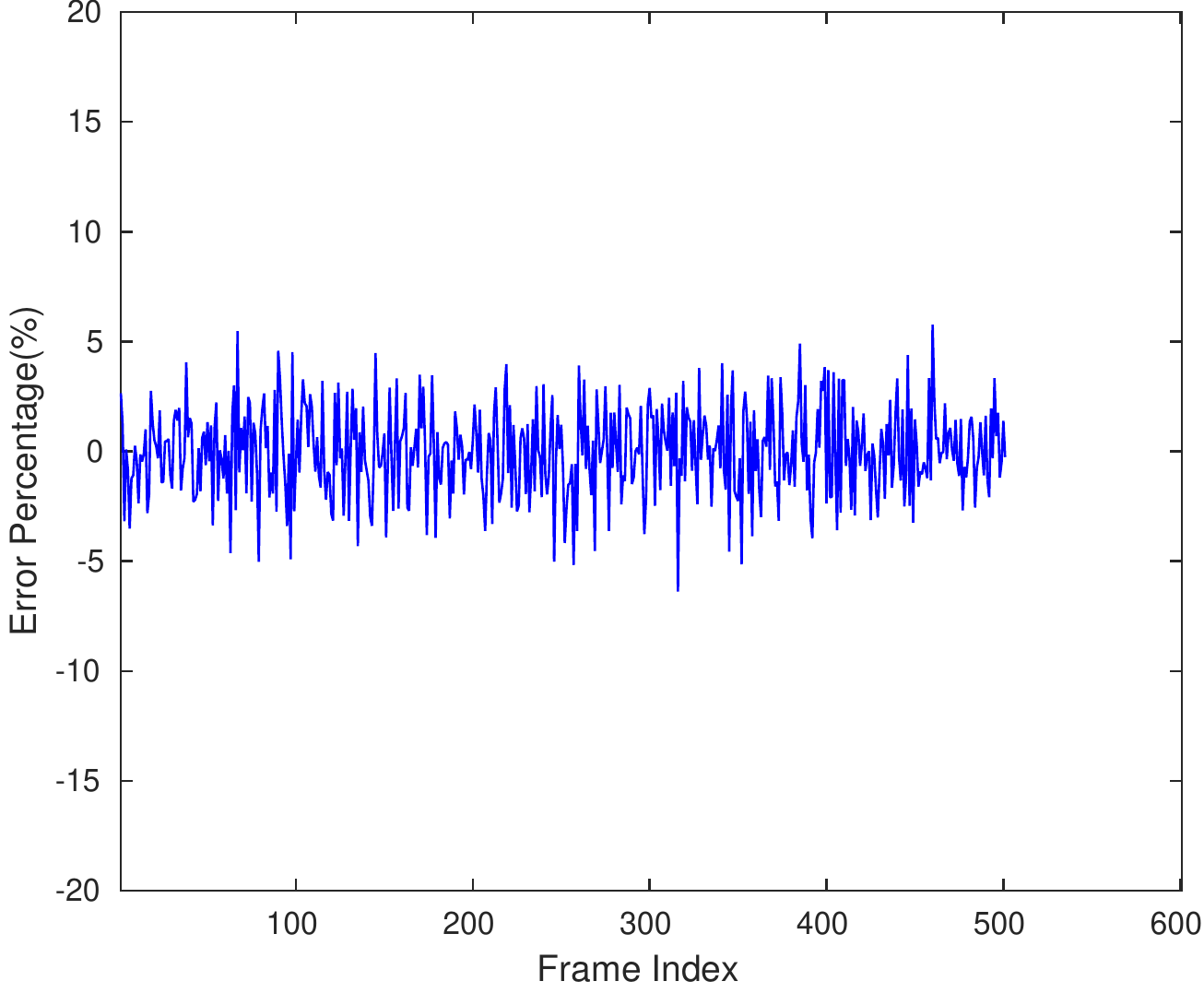}} 
\subfloat[]{\includegraphics[width=0.19\linewidth]{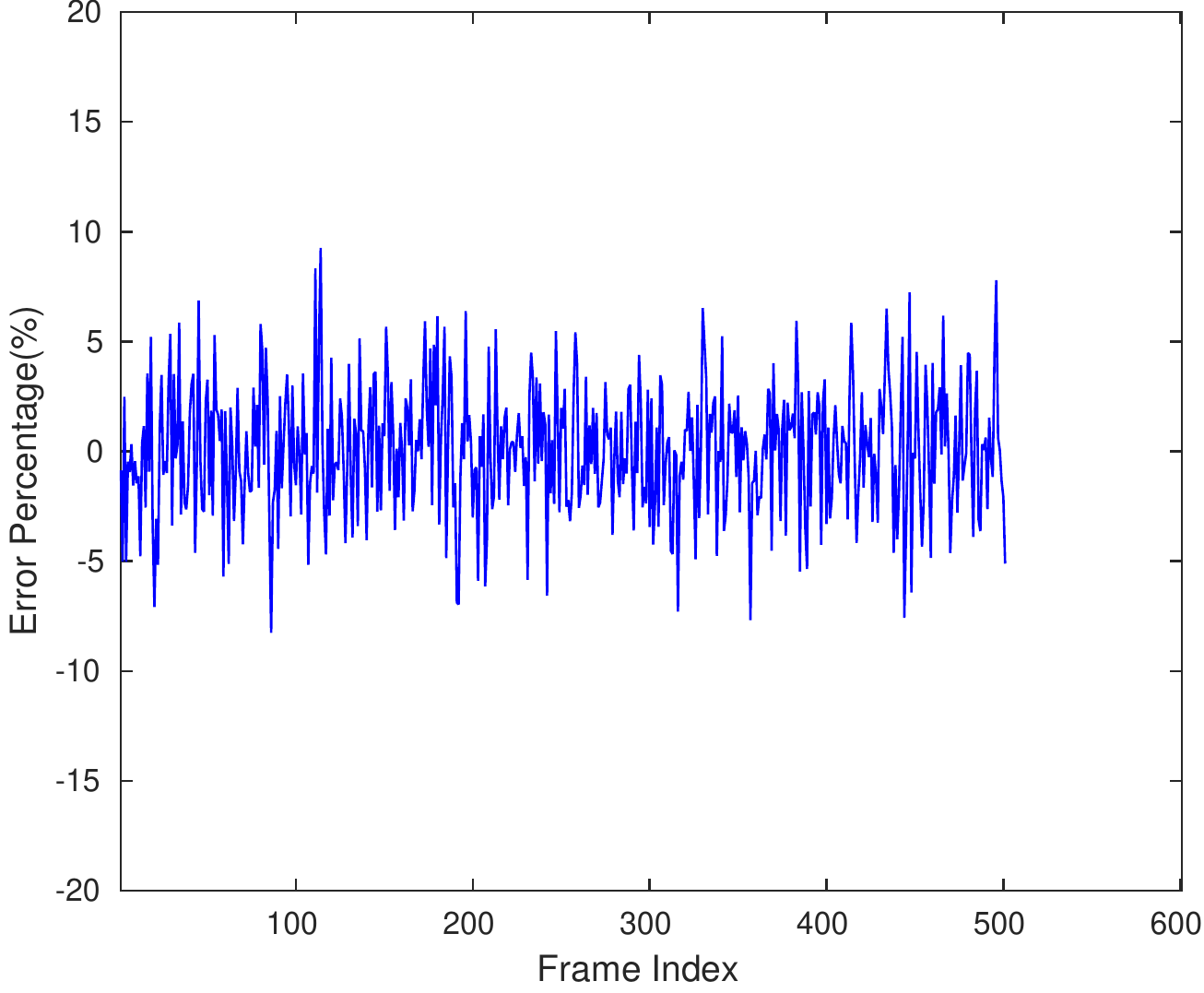}} 
\subfloat[]{\includegraphics[width=0.19\linewidth]{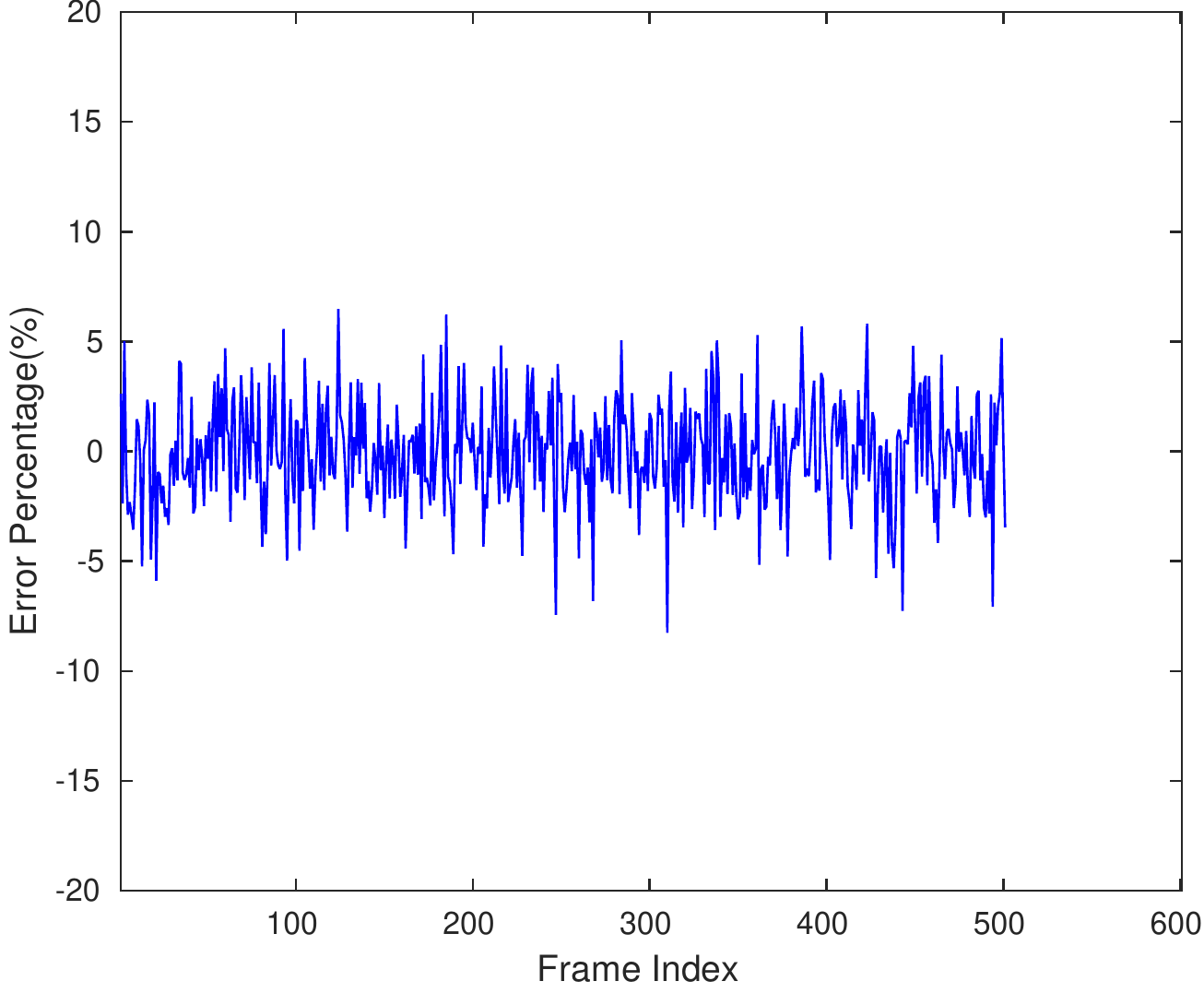}} 
\subfloat[]{\includegraphics[width=0.19\linewidth]{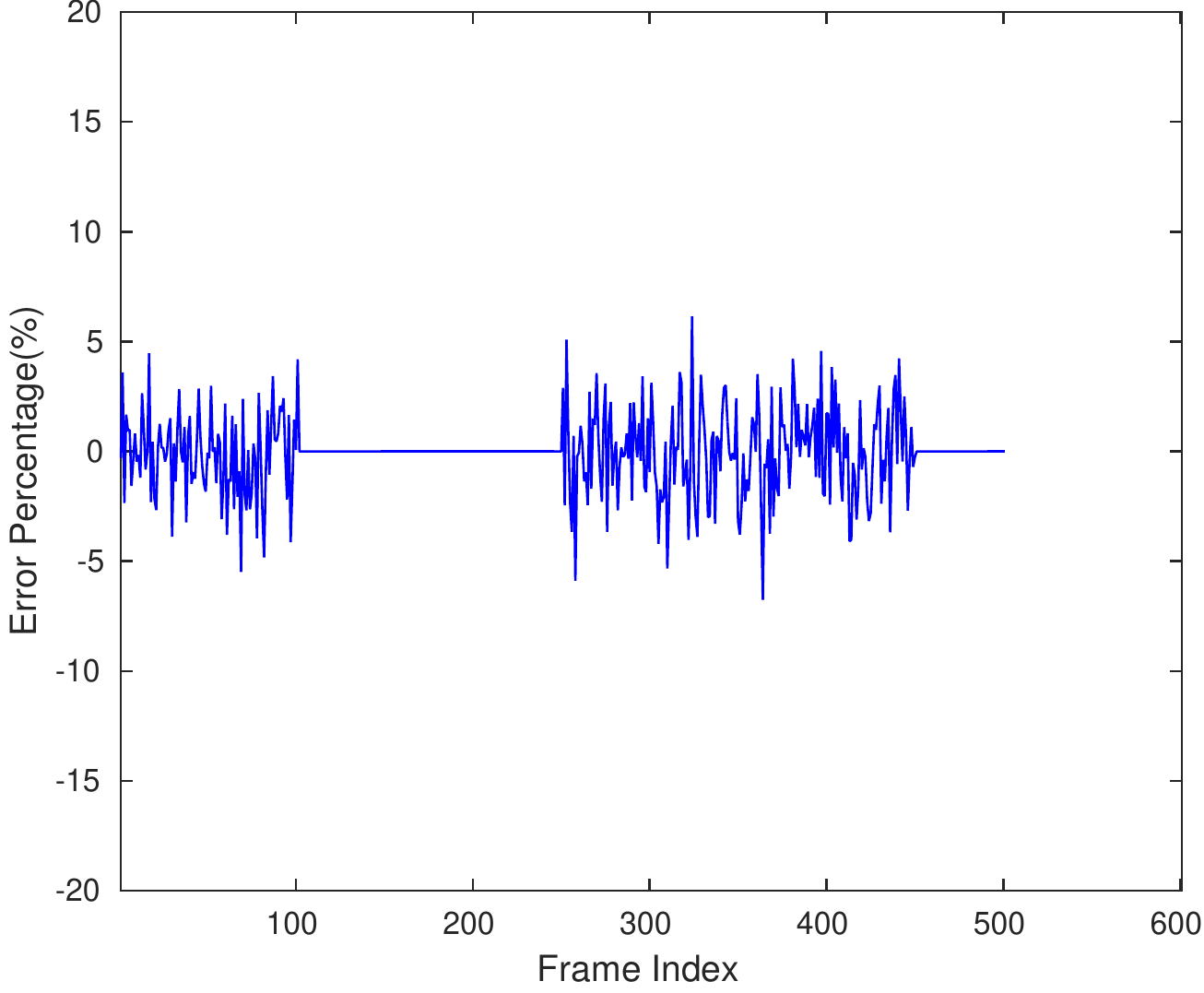}} \\
\subfloat[]{\includegraphics[width=0.19\linewidth]{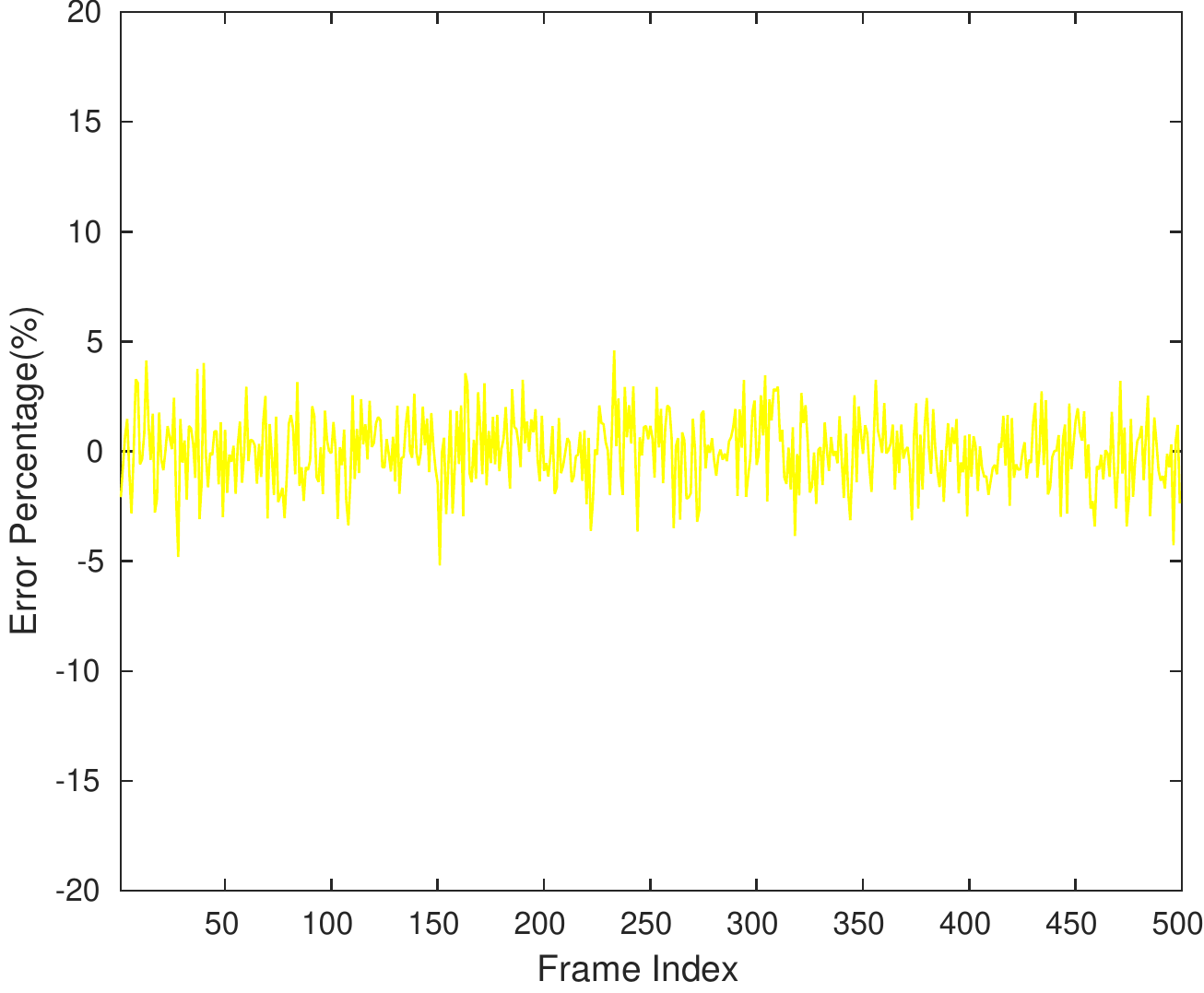}} 
\subfloat[]{\includegraphics[width=0.19\linewidth]{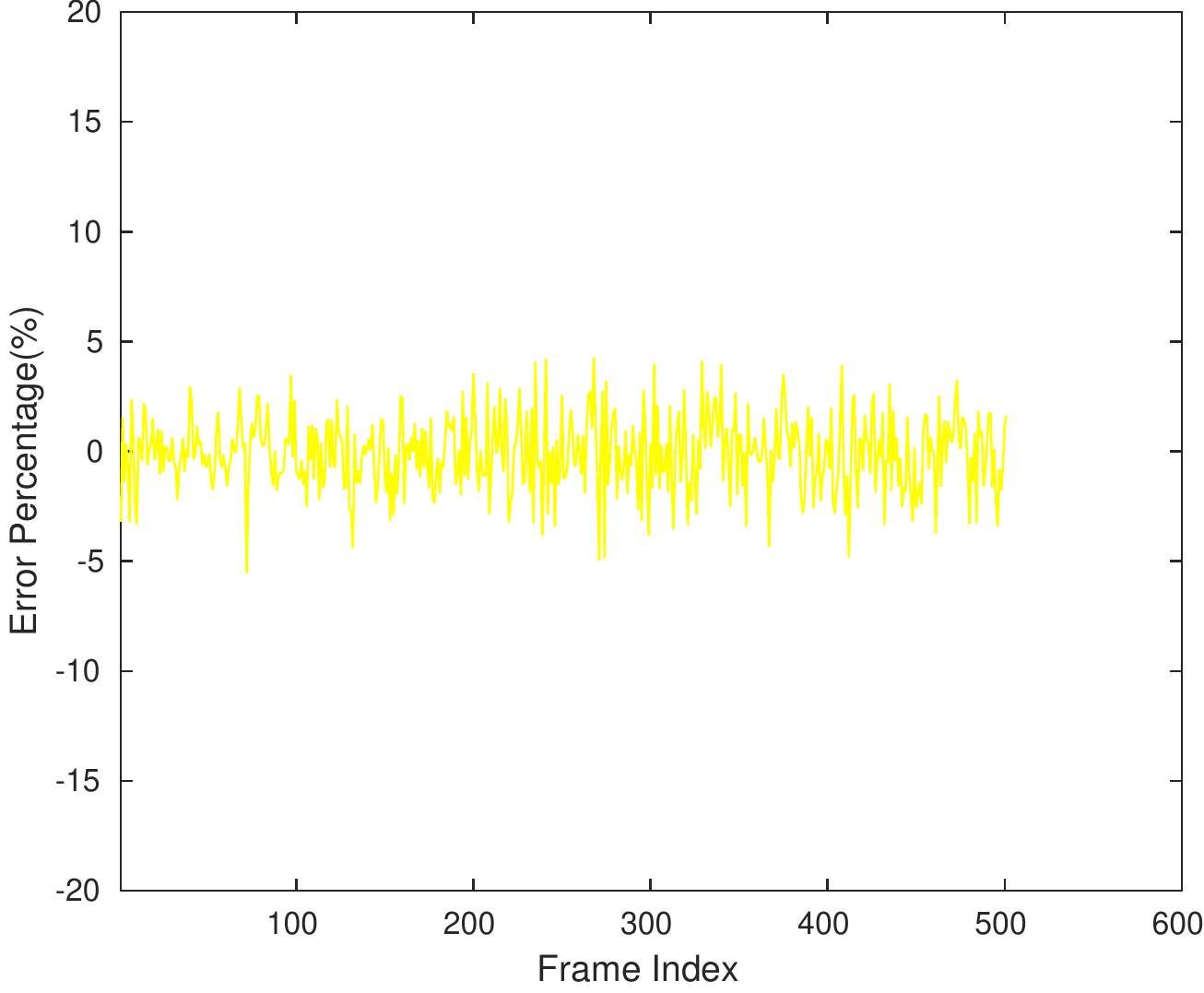}} 
\subfloat[]{\includegraphics[width=0.19\linewidth]{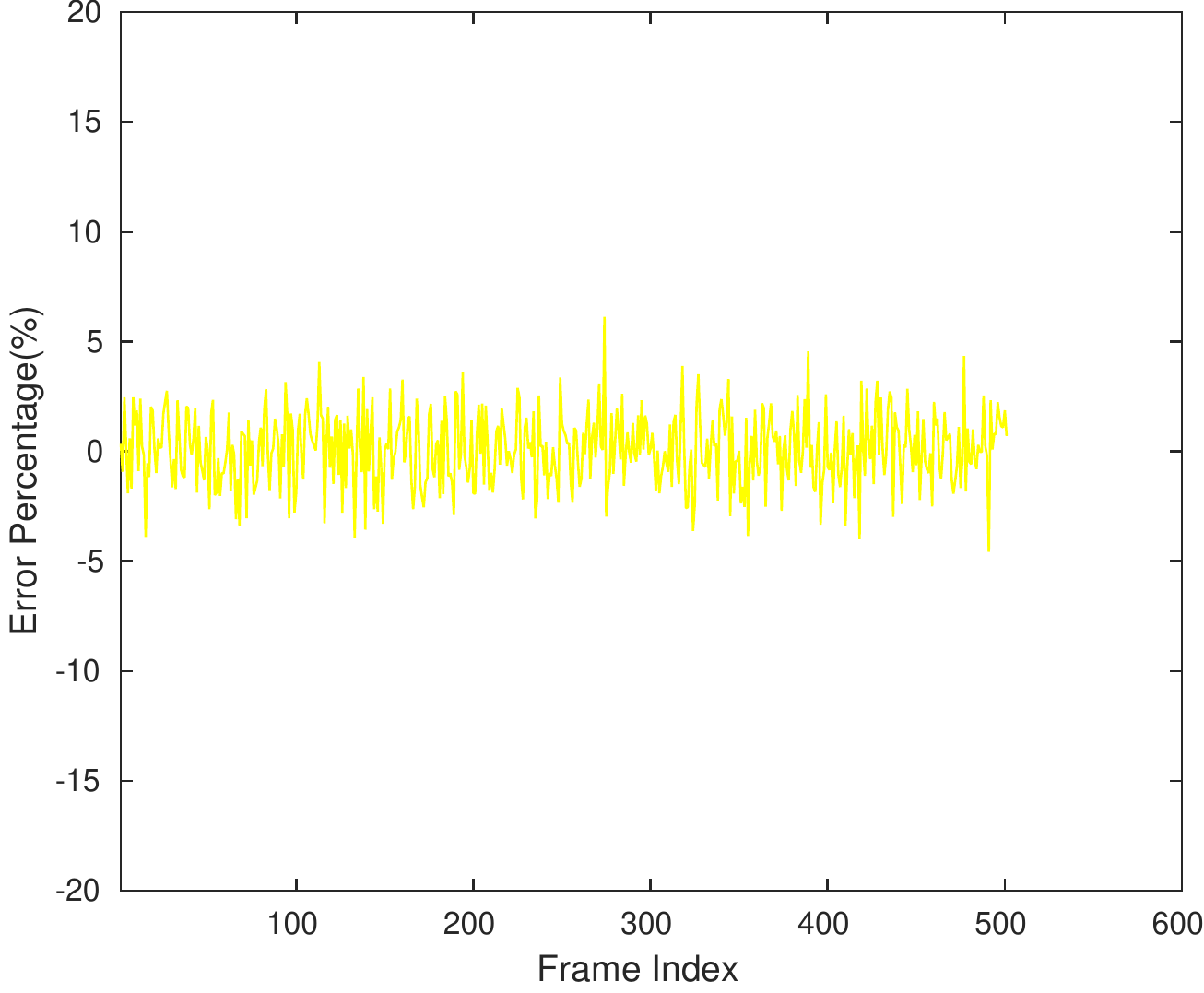}} 
\subfloat[]{\includegraphics[width=0.19\linewidth]{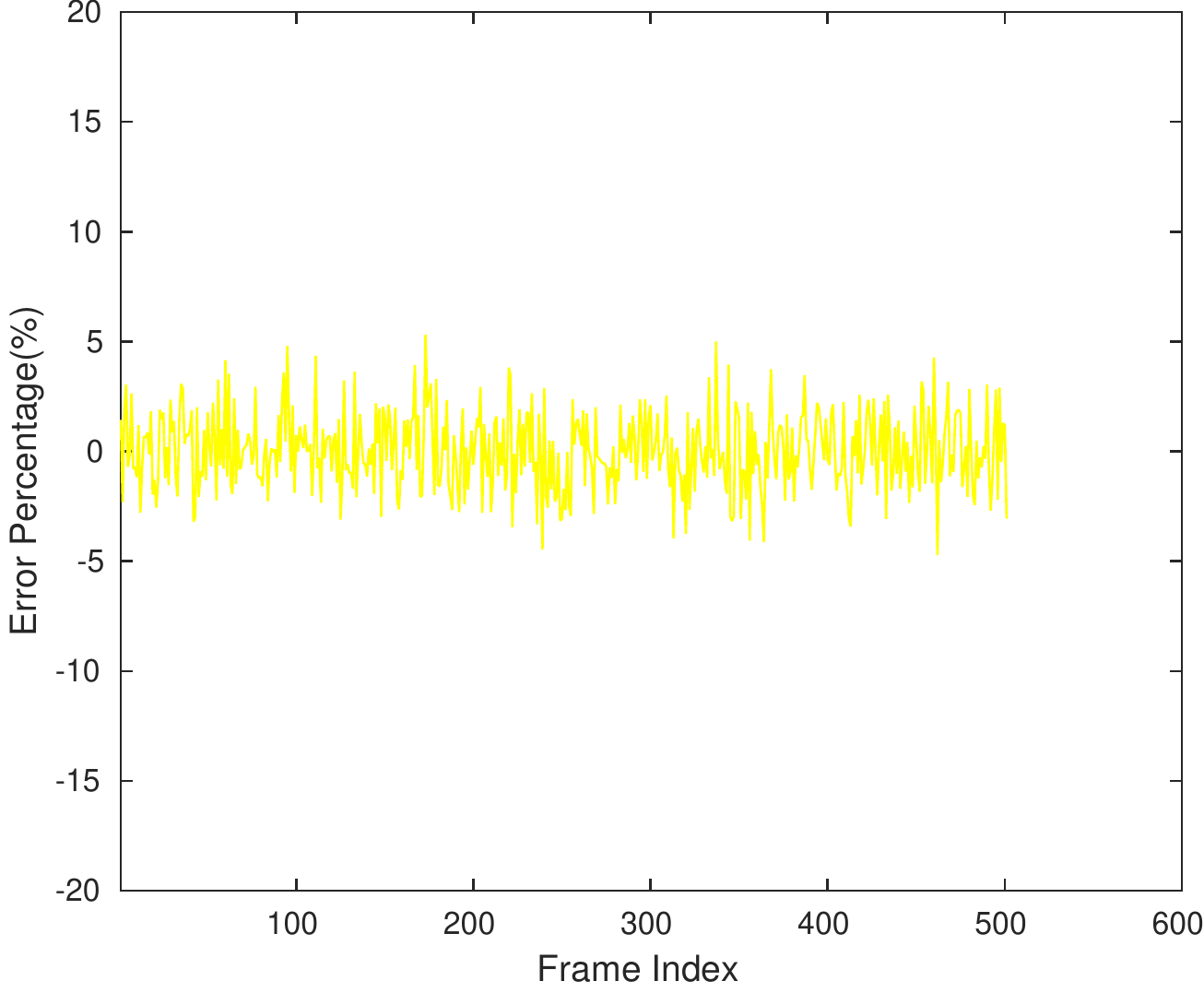}} 
\subfloat[]{\includegraphics[width=0.19\linewidth]{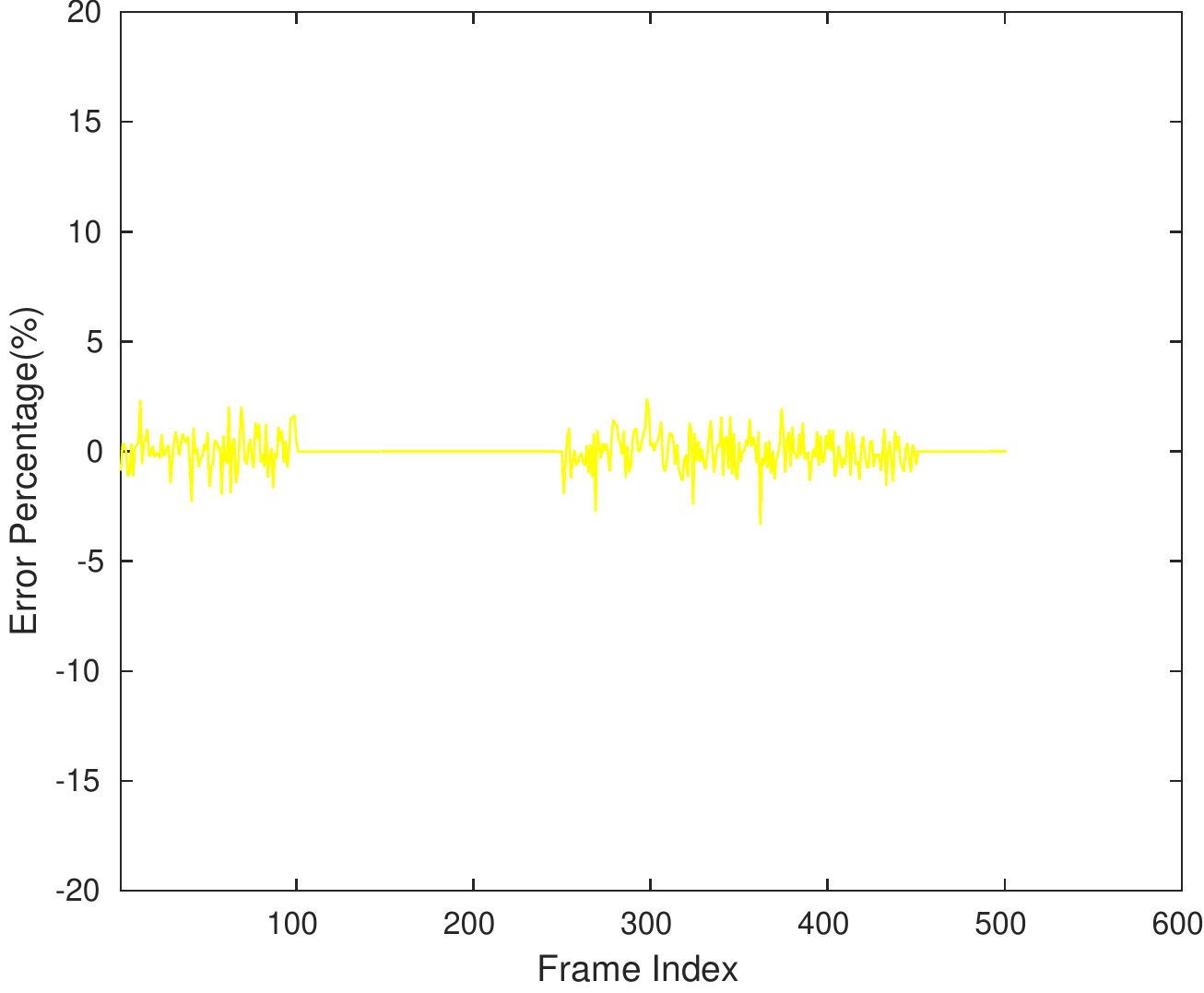}} \\

\caption{Error comparison among our method with previous approaches. For general PD, we compare our method (4th row, blue) and the active learning approach~\cite{Pan:2013:EPD} (2nd row, red). For translational PD, we compare our method (5th row, yellow) and the local optimization approach~\cite{Tang:2014:IGP,Je:2012:PRP} (3rd row, cyan). The comparison is performed on five different benchmarks (from left to right: CAD1, Bunny, Lion, Dragon and Donuts). These errors are computed according to the ground-truth (the first row) computed by a Minkowski sum approach~\cite{Lien:2009:ASM}. The $x$-axis is the frame index and the $y$-axis is the error magnitude. We observe a smaller error in our approach than other two approaches: our error is 50-80 times smaller than~\cite{Tang:2014:IGP,Je:2012:PRP} and~\cite{Pan:2013:EPD}.}
\label{fig:comparison}
\end{figure*}

\begin{small}
\begin{table}
\centering{}
\begin{tabular}{|c|c|c|c|c|c|c|c|c|}
 \hline
{}& Donut & CAD1  &
CAD2 & Dragon & Teeth & Buddha \tabularnewline
\hline
 Our ($PD_g$) &{0.93}  &{0.95} &{1.03} &{1.18} &{0.98} &{1.31}\tabularnewline
\hline
 Our($PD_t$) &{0.81}  &{0.85} &{1.13} &{1.06} &{1.01} &{1.11}\tabularnewline
\hline
Active \protect\cite{Pan:2013:EPD} &{0.95}  &{1.41} &{1.32} &{1.43}&{0.92} &{1.48} \tabularnewline
\hline
Local \protect\cite{Tang:2014:IGP} &{6.56}  &{78.3} &{54.4} &{132}&{111} &{153} \tabularnewline
\hline
\end{tabular}
\caption{\footnotesize{Comparison of the run-time cost for a single PD query (in ms) between our method and two prior methods: active learning~\protect\cite{Pan:2013:EPD} and local optimization~\protect\cite{Tang:2014:IGP,Je:2012:PRP}. Our method uses $100,000$ random-samples for precomputation, and the time cost is averaged over $1,000$ randomly generated in-collision queries. }}
\label{table:PDQuery}
\vspace*{-0.3in}
\end{table}
\end{small}

\section{Analysis}
\label{sec:analysis}
In this section, we evaluate some properties of the propagation sampling algorithm used in our precomputation phase. In particular, we discuss 1) the bounds on the time complexity of our precomputation scheme, and 2) why our approach can generate samples with better distribution in the contact space than pure random sampling.

\subsection{Time Complexity}
For the time complexity analysis, we assume the object $A$ can only perform translational movements and has $m$ vertices, and object $B$ is a connected mesh with $n$ vertices, where $m \ll n$. We also denote $T_{DCD}$ and $T_{CCD}$ as the time costs for one continuous collision checking and one discrete collision checking respectively. On average, $T_{CCD} \approx 7.46 T_{DCD}$ in our 10 thousand random tests. Then we have the following estimation for the precomputation's time complexity:
\begin{theorem}
The precomputation's time complexity has a lower bound of $T_{CCD} + (n-1)T_{DCD}$, and has an upper bound of $n \lg (n) T_{CCD}$.
\label{thm:faster}
\end{theorem}
\begin{proof}
To obtain a high quality sample-based representation for the contact space, one way is to generate all the $n$ configurations where $A$ contacts with $B$ at one of $B$'s vertices. Ideally, we only need to generate one random-sample, and all the other configurations can be visited in one iteration of propagation process. The time complexity for such ideal case is $T_{CCD} + (n-1)T_{DCD}$, which is a lower bound of the propagation sampling's time complexity.

The upper bound of the propagation sampling can be achieved by considering the expected time cost of a pure random sampling required to visit all the $n$ vertices of object $B$. Suppose the sampling process has already visited $i$ vertices of $B$, the expected time for the sampling process to find a new vertex different from the visited $i$ vertices is $\frac{n}{n-i} T_{CCD}$. As a result, the expected time cost for the entire sampling procedure is $\sum_{i=1}^{n-1} \frac{n}{n-i} T_{CCD} = n \lg (n) T_{CCD}$, which can serve as a upper bound of our propagation sampling's time complexity.
\end{proof}

The two bounds estimated in this theorem are conservative, but intuitively describes the performance of our precomputation scheme. In practice, the running time of our algorithm is much lower than the upper bound estimated, because for each propagation iteration started from a random-sample, we can generate a large number of propagate-samples due to our special treatment to the internal configuration cases in Section~\ref{sec:method}.

\subsection{Sample Distribution}
We next show that our method can generate samples evenly distributed over the contact space.
First, the propagation movement has the following property:
\begin{theorem}
\label{them:commutativity}
\textbf{Commutativity:} Suppose the object $A$ slides along the surface of the object $B$ according to the transition function $(\q^1, p_A, p_B^1, \theta) = \mathcal{T}(\q^0, p_A, p_B^0, \theta)$, where the transition function is as described in Equation~\ref{eq:transit}. $\q^0$ and $\q^1$ are object $A$'s configurations before and after the transition, and $p_B^0$ and $p_B^1$ are the contact points on the object $B$ for these two configurations. Then the inverse slide can be formulated by the transition function $(\q^0, p_A, p_B^0, \theta) = \mathcal{T}(\q^1, p_A, p_B^1, \theta)$.
\end{theorem}

\begin{proof}
This property is due to the fact according to our transition rule, the state and goal states of the transition function have one-to-one correspondence. Intuitively speaking, this means that if the sample $\q^0$ is visited first, then it can extend to $\q^1$ using the transition $\mathcal{T}$; if the sample $\q^1$ is visited first, it can extend to $\q^0$ using the same transition.   
\end{proof}

Next, we show that the samples generated by our propagation scheme will have no duplications:
\begin{theorem}
\label{them:Unique}
\textbf{Uniqueness:} Any two sets of samples generated by different propagation iterations will have no overlaps.
\end{theorem}
\begin{proof}
We use Figure~\ref{fig:analysis3} as an example for the two sets of sample configurations, where points belonging to different sets are marked as orange and blue respectively. Each sample set includes one seed random-sample (point $0$ and $7$ for this example), and a set of other samples extended from the seeds. 

If the two sample sets indeed overlap with each other, and we assume the duplicated points are $6$ and $10$, then according to Theorem~\ref{them:commutativity}, the point $6$ (also $10$) is able to propagate to $9$, then $7$ and $8$, and finally $11$. This means that all the points in the second sample set will overlap with some points in the first set. However, while generating the random-sample for the second set, we use a Kd-tree to guarantee that its distance to all points in the first set is large enough and thus such duplication for the random-sample is impossible. 
\end{proof}
This theorem explains why our approach will not generate repeated samples and can result in an efficient precomputation phase. 

\begin{figure}[!ht]
  \centering
  \includegraphics[width=0.8\linewidth]{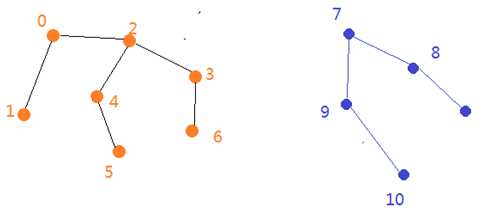}
  \caption{Any two sets of samples generated by different propagation iterations will have no overlaps.}
  \label{fig:analysis3}
\end{figure}

Overall, our propagation sampling method has three benefits: 
\begin{itemize}
\item \textbf{Efficiency} Theorems \ref{them:commutativity} and \ref{them:Unique} imply that no duplicated samples would be generated during the propagation. As a result, the algorithm will not waste time on generating redundant samples, and can make steady progress toward constructing a more precise representation of the contact space while generating more and more samples. These properties result in the error reduction shown in Figure~\ref{fig:converge}. In addition, these two theorems also guarantee the even distribution of samples over the surface of the contact space. 
\item \textbf{Accuracy} Our propagation steps directly slide the moving object over the surface of the fixed objects. As a result, every generated sample locates exactly on the contact space, and this results in the high accuracy of the PD results. 
\end{itemize}

\section{Limitations, Conclusions and Future Work}
We present a new PD approximation algorithm between general 3D models. We compute an approximation of contact space using propagation sampling. The propagation sampling scheme improves the accuracy of the approximation and is much faster as compared to only using random samples. We highlight the performance on many complex 3D models and highlight the benefits in terms of runtime performance and accuracy.

Our approach has some limitations. If the contact space has very narrow components, our sampling approach may miss them and thereby affect the accuracy of PD computations. It is possible that propagation sampling may not find samples during the local search computation. We would like to investigate the use of narrow-passage algorithms in sample-based motion planning to improve the performance. It would be useful to improve the performance of propagation-sampling by utilizing the local curvature of the surface to choose appropriate directions. We would like to evaluate the performance on complex and articulated models, and also integrate with dynamics simulation and motion planning algorithms.

{\small
\bibliographystyle{IEEEtran}
\bibliography{references}
}

\end{document}